\documentclass[letterpaper, 10 pt, conference]{ieeeconf}

\IEEEoverridecommandlockouts                              

\usepackage{textcomp}

\usepackage{graphics,graphicx} 
\usepackage{epsfig} 
\usepackage{mathptmx} 
\usepackage{times} 
\usepackage{amsmath} 
\usepackage{amssymb,nth, mathtools,dsfont}  
\usepackage{amsbsy}
\usepackage{epstopdf}
\usepackage[noadjust]{cite}
\usepackage{romannum}
\usepackage{xcolor}
\usepackage{subcaption}
\usepackage{multirow}

\newcommand{\bmtx}{\begin{bmatrix}}
\newcommand{\emtx}{\end{bmatrix}}
\newcommand{\bsmtx}{\left[ \begin{smallmatrix}} 
\newcommand{\esmtx}{\end{smallmatrix} \right]}
\newcommand{\field}[1]{\mathbb{#1}}
\newcommand{\R}{\field{R}}
\newcommand{\Sym}{\field{S}}

\newcommand{\N}{\field{N}}

\newcommand{\Dpm}{D_{\pm 1}}

\newcommand{\Mset}[1]{\mathcal{M}_{#1}}
\newcommand{\Mc}{\mathcal{M}_{c}}
\newcommand{\Miset}[1]{\mathcal{M}^{inc}_{#1}}
\newcommand{\Mic}{\mathcal{M}_c^{inc}}

\newcommand{\fflip}[1]{{#1}^{\mbox{\footnotesize flip}}}

\newtheorem{theorem}{Theorem}
\newtheorem{corollary}{Corollary}
\newtheorem{lemma}{Lemma}
\newtheorem{definition}{Definition}

\newtheorem{proposition}{Proposition}

\makeatletter
\newtheorem{repeatlem@}{Lemma}
\newenvironment{repeatlem}[1]{%
    \def\therepeatlem@{\ref{#1}}
    \repeatlem@
}
{\endrepeatlem@}
\makeatother

\title{A Complete Set of Quadratic Constraints for 
\\
Repeated ReLU and Generalizations}

\author{Sahel Vahedi Noori, Bin Hu, Geir Dullerud, and Peter Seiler
	\thanks{S. Vahedi Noori and P. Seiler are with the Department of Electrical Engineering \& Computer Science, at the University of Michigan ({\tt\small sahelvn@umich.edu;} and
		{\tt\small pseiler@umich.edu}). B. Hu is with the Department of Electrical and Computer Engineering at the University of Illinois at Urbana-Champaign ({\tt \small binhu7@illinois.edu}). G. Dullerud is with the
  the Department of Mechanical Science and Engineering at the University of Illinois at Urbana-Champaign ({\tt \small dullerud@illinois.edu})
  }
	}

\begin{document}

\maketitle

\begin{abstract}
This paper derives a complete set of quadratic constraints (QCs) for the  repeated ReLU.  The complete set of QCs is described by a collection of matrix copositivity conditions. We also show that only two functions satisfy all QCs in our complete set:  the repeated  ReLU and flipped ReLU.   Thus our complete set of QCs bounds the  repeated ReLU as tight as possible up to the sign invariance inherent in quadratic forms. We derive a similar complete set of incremental QCs for repeated  ReLU, which can potentially lead to less conservative Lipschitz bounds for ReLU networks than the standard LipSDP approach. 
The basic constructions are also used to derive the complete sets of QCs for other piecewise linear activation functions such as leaky ReLU, MaxMin, and HouseHolder. Finally, we illustrate the use of the complete set of QCs to  assess stability and performance for recurrent neural networks with ReLU activation functions. We rely on a standard  copositivity relaxation to formulate the stability/performance condition as a semidefinite program. Simple examples are provided to illustrate that the complete sets of QCs and incremental QCs can yield less conservative bounds than existing sets. 



\end{abstract}


\section{Introduction}

This paper considers the use of quadratic constraints (QCs) to characterize a nonlinear function. QCs are inequalities expressed as quadratic forms on the input/output graph of the function.  These constraints are useful as they can be easily incorporated into stability and performance conditions for dynamical systems.  This is a special case of the more general integral quadratic constraint (IQC) framework \cite{megretski97,veenman16,scherer22,seiler15}.  Moreover, a related class of incremental QCs can be used to compute Lipschitz bounds on static functions \cite{fazlyab2019efficient,wang2022quantitative,revay2020lipschitz,araujo2023unified,wang2023direct,havens2024exploiting,fazlyab2024certified,wang2024scalability,pauli2024novel,pauli2023lipschitz,pauli2024lipschitz}.

This paper focuses on the specific nonlinearity known as the  Rectified Linear Unit (ReLU), and its characterization in terms of QCs and incremental QCs. This special case is motivated by the popular use of the  ReLU as the activation function in neural networks (NNs) and recurrent neural networks (RNNs) \cite{nair2010rectified,krizhevsky2012imagenet,chung2014empirical}. The scalar ReLU $\phi:\R\to\R$ 
and the more general (repeated) ReLU $\Phi:\R^{n_v}\to \R^{n_v}$
are defined formally in Section~\ref{sec:bgReLUQCs}.
The scalar ReLU satisfies a number of useful properties including positivity
($\phi(v)\ge 0$), positive complement ($\phi(v)\ge v$), complementarity ($\phi(v)( v-\phi(v))=0$), and slope restrictions.   These properties have been previously leveraged to derive several useful QCs for the repeated ReLU $\Phi$~\cite{richardson23,drummond24,fazlyab2020safety,ebihara21EJC,ebihara21CDC}. 

QCs can be combined with Lyapunov/dissipativity theory 
\cite{willems72a,willems72b,schaft99,khalil01} 
to derive stability and performance conditions for RNNs.
Related work along these lines for discrete-time RNNs is given in
\cite{soykens99,tan2024robust,yin2021stability,richardson23,ebihara21EJC,ebihara21CDC,noori24ReLURNN}. In addition, incremental QCs have been used to compute Lipschitz bounds on NNs 
\cite{fazlyab2019efficient}.  There is a rapidly growing literature on the application of incremental QCs including extensions for $\ell_\infty$ perturbations
\cite{wang2022quantitative}, deep equilibrium models \cite{revay2020lipschitz}, and the use for designing Lipschitz networks \cite{araujo2023unified,wang2023direct,havens2024exploiting,fazlyab2024certified}.  There is also recent work on scalability to ImageNet
\cite{wang2024scalability} as well as extensions to convolutional structures \cite{pauli2024novel,pauli2023lipschitz} and 
activation functions such as MaxMin and GroupSort
\cite{pauli2024lipschitz}.

The following summarizes the key contributions of our paper to this existing literature:
\begin{enumerate}
\item  We provide the complete class of QCs satisfied by the repeated ReLU ($\Mc$ in Theorem~\ref{thm:ReLUQCs1}). This complete set of QCs  is described by a collection of $2^{n_v}$ matrix copositivity conditions where $n_v$ is the dimension of the repeated ReLU.  Importantly,  we also show that only two functions satisfy all QCs in our complete set:  the repeated ReLU and flipped ReLU 
(Theorem~\ref{thm:ReLUQCs2}). Thus our complete set of QCs bounds the repeated ReLU as tight as possible up to the sign invariance inherent in quadratic forms.  
\item We derive the complete set of incremental QCs for repeated ReLU ($\Mic$ in Theorem~\ref{thm:ReLUIncQCs}). We further demonstrate the utility of this result by using it to derive a new subclass of incremental QCs for repeated ReLU ($\Miset{2}$ in Lemma~\ref{thm:IncQC}).
\item We show in Section~\ref{sec:relatedresults} how to adapt the results to derive the complete class of QCs for other important piecewise linear activation functions including the leaky ReLU, MaxMin, and Householder activations.
\item Finally, Section~\ref{sec:QCstab} applies the complete set  of QCs for  analysis of RNNs with ReLU activation functions.  A stability and performance condition is derived using standard Lyapunov/dissipativity arguments. There are computational issues associated with the use of the complete set due to the copositivity conditions.  Thus we rely on an existing copositivity relaxation to formulate our stability/performance condition as a semidefinite program (SDP). 
\end{enumerate}

Two simple examples illustrating the approach are given in Section~\ref{sec:Examples}. The first example demonstrates that the complete class of QCs $\Mc$ can provide less conservative bounds on the induced $\ell_2$ gain of an RNN.  The second example demonstrates that the the complete class of incremental QCs $\Mic$ can provide less conservative Lipschitz bounds on a simple NN. Both examples are academic in nature but demonstrate that improved bounds are possible beyond the typical sets of QCs and incremental QCs used in the literature for repeated ReLU.

\section{Background}
\label{sec:bg}

\subsection{Notation}
\label{sec:notation}

This section introduces basic notation regarding vectors, matrices, and signals.  Let $\R^n$ and $\R^{n\times m}$ denote the sets of real $n\times 1$ vectors and $n\times m$ matrices, respectively. Moreover, $\R_{\ge 0}^n$ and $\R_{\ge 0}^{n\times m}$ denote vectors and matrices of the given dimensions with non-negative entries. Thus $\R_{ \ge 0}^n$ corresponds to the non-negative orthant. We let
$e_i \in \R^{n_v}$ denote the $i^{th}$ standard basis vector; every entry of $e_i$ is zero except entry $i$ is equal to 1.

Next, define the following sets of matrices:
\begin{itemize}
\item $\Sym^n$ is the set of real symmetric, $n\times n$ matrices.
\item $D^n$ is the set of $n\times n$, real diagonal matrices.
\item $COP^n$ is the set of $n\times n$, real, symmetric co-positive matrices. Specifically, if $M=M^\top\in COP^n$ then
$x^\top M x \ge 0$ for all $x\in \R^{n}_{\ge 0}$.
\item $DH^n$ is the set of $n\times n$, real, doubly hyperdominant matrices. Such matrices have non-positive off-diagonal elements while both the row sums and column sums are non-negative.
\item $\N$ is the set of non-negative integers.  
\end{itemize}
For $M\in\Sym^n$, we use $M\succeq 0$ and $M\succ 0$  to denote that it is positive semidefinite or positive definite, respectively.

Finally, let $v:\N \to \R^n$ and $w:\N \to \R^n$ be real, vector-valued
sequences.  Define the inner product $\langle v,w \rangle : = \sum_{k=0}^\infty v(k)^\top w(k)$.  The
set $\ell_2^n$ is an inner product space with sequences $v$ that satisfy $\langle v,v\rangle < \infty$. The corresponding norm is
$\|v\|_2 := \sqrt{ \langle v,v \rangle}$. 

\subsection{QCs for ReLU}
\label{sec:bgReLUQCs}

The scalar Rectified Linear Unit (ReLU) is a function $\phi:\R \to \R_{\ge 0}$
defined by:
\begin{align}
\label{eq:ReLU}
\phi(v)= \left\{
  \begin{array}{ll}
    0 & \mbox{if } v < 0 \\
    v & \mbox{if } v \geq 0 
  \end{array} 
  \right. .
\end{align}
The graph of the scalar ReLU is shown in Figure~\ref{fig:ScalarReLU}.
The repeated ReLU is the function   $\Phi:\R^{n_v} \to \R_{\ge 0}^{n_v}$  defined by
$\Phi(v)= (\phi(v_1), \, \phi(v_2), \ldots , \, \phi(v_{n_v}))$.  In the remainder of the paper we will use
"ReLU" to refer to the repeated ReLU and the term "scalar ReLU" when we want to emphasize the case with a single input and output.

\begin{figure}[h!t]
\centering
\begin{picture}(100,90)(170,25)
 \thicklines
 \put(170,45){\vector(1,0){100}}  
 \put(260,50){$v$}
 \put(220,25){\vector(0,1){70}}  
 \put(205,100){$w=\phi(v)$}
 \put(170,46){{\color{red} \line(1,0){50}} }  
 \put(220,46){{\color{red} \line(1,1){35}} }  
\end{picture}
\caption{Graph of scalar ReLU $\phi$.}
\label{fig:ScalarReLU}
\end{figure}
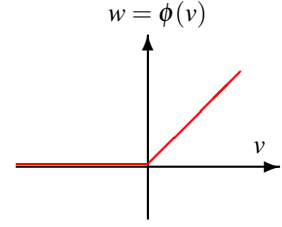

This paper mainly focuses on QCs satisfied by the ReLU. The notion of a QC is formally defined next.

\vspace{0.1in}
\begin{definition}
A function $F:\R^{n_v} \to \R^{n_v}$ satisfies the \underline{Quadratic Constraint (QC)} defined by $M\in \Sym^{2n_v}$ if the
following inequality holds for all $v\in \R^{n_v}$:
\begin{align}
\label{eq:QCdef}
\bmtx v \\ F(v)\emtx^\top 
M
\bmtx v \\ F(v)\emtx \ge 0 .
\end{align}
We denote by $QC(M)$ the set of all functions which satisfy the QC defined by $M\in \Sym^{2n_v}$.  Furthermore, given a subset of matrices $\mathcal{M}\subset\Sym^{2n_v}$ we  define
$QC(\mathcal{M}):= \cap_{M\in\mathcal{M}} QC(M)$; namely, $F\in QC(\mathcal{M})$ means $F$ satisfies every QC defined by the elements of $\mathcal{M}$. 
\end{definition}
\vspace{0.1in}

QCs are, in general, conservative bounds on the graph of a function in the sense that both $QC(M)$ and  $QC(\mathcal{M})$ will typically contain more than one function. However, these bounds are useful as they can be easily incorporated into stability and performance conditions for dynamical systems.  Section~\ref{sec:QCstab} will provide one such application of QCs for stability and performance analysis.

QCs for ReLU can be derived using simple properties of the scalar ReLU.  Specifically, the
scalar ReLU satisfies the following useful properties that have been used previously in the literature \cite{richardson23,drummond24,fazlyab2020safety,ebihara21EJC,ebihara21CDC}:
\begin{enumerate}
\item \emph{Positivity:} The scalar ReLU is non-negative for all inputs: $\phi(v)\ge 0$ $\forall v\in \R$.

\item \emph{Positive Complement:} The scalar ReLU satisfies $\phi(v)\ge v$ $\forall v\in \R$.

\item \emph{Complementarity:} The graph of the scalar ReLU is identically on the line of slope 0 (when $v\le 0$) or the line of slope 1 (when $v\ge 0$).  Thus, $\phi(v) \, (v-\phi(v))=0$ $\forall v\in\R$.

\item \emph{Positive Homogeneity:} The scalar ReLU is homogeneous for all non-negative constants: $\phi(\beta v)= \beta \phi(v)$ $\forall v\in \R$ and $\forall \beta \in \R_{\ge 0}$.
\end{enumerate}
We can express equivalent properties for ReLU, e.g. the positivity and positive complement properties are $\Phi(v) \ge 0$ and $\Phi(v)\ge v$ for all $v\in \R^{n_v}$, respectively.  
The next lemma gives two QCs for the  ReLU based on these properties. Both QCs are known in the literature and the second one, defined by $\Mset{2}$  below, is a restatement of Theorem 2 in \cite{ebihara21CDC}.

\vspace{0.1in}
\begin{lemma}
    Define the following subsets of $\Sym^{2n_v}$:
    \begin{align}
    \label{eq:M1}
    \Mset{1} & := \left\{ \bmtx 0 & Q_1 \\ Q_1 & -2Q_1 \emtx \, : \, Q_1 \in D^{n_v} \right\},  \\
    \label{eq:M2}
    \Mset{2} & := \left\{ \bmtx -I & I \\ 0 & I \emtx^\top  Q_2 \bmtx -I & I \\ 0 & I \emtx
     \, : \,  Q_2 \in COP^{2n_v} \right\}.
    \end{align}
    The  ReLU $\Phi:\R^{n_v} \to \R_{\ge 0}^{n_v}$
    satisfies the QCs defined by any matrices in $\Mset{1}$ and $\Mset{2}$, i.e.
    $\Phi\in QC(\Mset{1})$ and $\Phi\in QC(\Mset{2})$.
\end{lemma}
\vspace{0.05in}
\begin{proof}
Take any $M_1\in \Mset{1}$ and then observe that
for all $v\in \R^{n_v}$ we algebraically have
\begin{align*}
\bmtx v \\ \Phi(v)\emtx^\top 
M_1 
\bmtx v \\ \Phi(v)\emtx = 
    2\sum_{i=1}^{n_v} (Q_1)_{ii} \phi (v_i) (\phi (v_i)-v_i).
\end{align*}
\noindent 
By complementarity of scalar ReLU the right-hand side is zero and thus $\Phi$ satisfies the QC defined by  any
$M_1 \in \Mset{1}$.

Next, choose any $M_2\in \Mset{2}$ and verify that for any $v\in \R^{n_v}$ we have
\begin{align*}
\bmtx v \\ \Phi(v) \emtx^\top 
M_2 
\bmtx v \\ \Phi(v) \emtx = 
\bmtx \Phi(v)-v \\ \Phi(v) \emtx^\top
Q_2
\bmtx \Phi(v)-v \\ \Phi(v) \emtx.
\end{align*}
The right-hand side is non-negative  because $Q_2$ is copositive and $\bsmtx \Phi(v)-v \\ \Phi(v) \esmtx$ is elementwise non-negative by the positivity and positive complement properties of $\phi$.
Thus $\Phi$ satisfies the QC defined by any
$M_2\in \Mset{2}$.
\end{proof}
\vspace{0.1in}

ReLU also satisfies QCs defined for the more
general class of slope restricted nonlinearities. Specifically, scalar ReLU is slope restricted to $[0,1]$, i.e.:
\begin{align}
\label{eq:ReLUslope}
  0 \le  \frac{ \phi(v_2)-\phi(v_1)}{v_2-v_1} \le 1, 
  \,\, \forall v_1, \, v_2 \in \R, \, v_1\ne v_2.
\end{align}
The next lemma defines a set of QCs for  ReLU based on this slope-restriction.  These are a special case of Zames-Falb multipliers \cite{Carrasco:2016}.
The lemma follows from existing results for slope restricted nonlinearities
(Section 3.5 of \cite{Willems:71} and
\cite{kulkarni02}\footnote{The results in these references are stated for   monotone nonlinearities. A similar fact holds for nonlinearities with slope restricted to $[0,1]$ by a transformation of the input-output data.}).

\vspace{0.1in}
\begin{lemma}
    \label{lem:ZamesFalb}
    Define the following subset of $\Sym^{2n_v}$:
    \begin{align}
    \label{eq:M3}
    \Mset{3}:=\left\{ \bmtx 0 & Q_3^\top \\ Q_3 & -(Q_3+Q_3^\top) \emtx
    \, : \, Q_3 \in DH^{n_v} \right\}.
    \end{align}
    The  ReLU $\Phi:\R^{n_v} \to \R_{\ge 0}^{n_v}$
    satisfies the QCs defined by any matrix in $\Mset{3}$, i.e.
    $\Phi\in QC(\Mset{3})$.
\end{lemma}
\vspace{0.1in}

Section~\ref{sec:CompleteQCs} gives an independent proof that  ReLU satisfies all QCs defined by $\Mset{3}$.

\subsection{Incremental QCs for  ReLU}
\label{sec:bgReLUIncQCs}

Incremental QCs, formally defined below, are a related class of quadratic constraints.  Incremental QCs can be used to compute Lipschitz bounds on Neural Networks (NNs), e.g. as in LipSDP \cite{fazlyab2019efficient} and related work \cite{wang2022quantitative,revay2020lipschitz,araujo2023unified,wang2023direct,havens2024exploiting,fazlyab2024certified,wang2024scalability,pauli2024novel,pauli2023lipschitz,pauli2024lipschitz}.

\vspace{0.1in}
\begin{definition}
A function $F:\R^{n_v} \to \R^{n_v}$ satisfies the \underline{incremental QC} defined by $M \in \Sym^{2n_v}$ if the
following inequality holds for all $\bar{v}, \, \hat{v}\in \R^{n_v}$:
\begin{align}
\label{eq:IncQCdef}
\bmtx \bar{v}-\hat{v} \\ F(\bar{v})-F(\hat{v})\emtx^\top 
M
\bmtx \bar{v}-\hat{v} \\ F(\bar{v})-F(\hat{v})\emtx
\ge 0.
\end{align}
We denote by $QC^{inc}(M)$ the set of all functions which satisfy the 
incremental QC defined by $M\in \Sym^{2n_v}$.  Furthermore, given a subset of matrices $\mathcal{M}\subset\Sym^{2n_v}$ we  define
$QC^{inc}(\mathcal{M}):= \cap_{M\in\mathcal{M}} QC^{inc}(M)$; namely, $F\in QC^{inc}(\mathcal{M})$ means $F$ satisfies every incremental QC defined by the elements of $\mathcal{M}$.

\end{definition}
\vspace{0.1in}

A set of incremental QCs for  ReLU can be derived from the $[0,1]$ slope restriction of scalar ReLU. 
Specifically, the slope restriction \eqref{eq:ReLUslope} implies that for any $\bar{v}$, $\hat{v} \in \R$  we have:
\begin{align}
\label{eq:ScalarReLUSlope}
(\phi(\bar{v})-\phi(\hat{v})) \cdot \left[ (\bar{v}-\hat{v}) - (\phi(\bar{v})-\phi(\hat{v})) \right] \ge 0.
\end{align}
It is well known that, based on this property, the  ReLU 
satisfies the incremental QCs defined by any matrices in the following set \cite{fazlyab2019efficient,pauli2024novel,pauli2023lipschitz}:
    \begin{align}
    \label{eq:Mi1}
    \Miset{1} & := \left\{ \bmtx 0 & T_1 \\ T_1 & -2T_1 \emtx \, : \, T_1 \in D^{n_v}, \, T_1 \succeq 0 \right\}.
    \end{align}
  To see this,  take any $M_1\in \Miset{1}$.
Then for any
$\bar{v}, \, \hat{v}\in \R^{n_v}$,
\begin{align*}
& \bmtx \bar{v}-\hat{v} \\ \Phi(\bar{v})-\Phi(\hat{v}) \emtx^\top 
M_{1}
\bmtx \bar{v}-\hat{v} \\ \Phi(\bar{v})-\Phi(\bar{v}) \emtx \\
& =
2 \sum_{i=1}^{n_v}  (T_1)_{ii} \cdot
(\phi(\bar{v})_i-\phi(\bar{v})_i) \cdot
\left[ (\bar{v}_i-\hat{v}_i) - (\phi(\bar{v})_i-\phi(\hat{v})_i)
\right].
\end{align*}

Each term of the sum is nonnegative due to $T_1\succeq 0$ and the $[0,1]$ slope restriction of scalar ReLU in \eqref{eq:ScalarReLUSlope}.
Thus $\Phi$ satisfies all incremental QCs defined by $\Miset{1}$, i.e. $\Phi \in QC^{inc}(\Miset{1})$. 
The incremental QCs defined by $\Miset{1}$ are standard and have been widely used in LipSDP and related work \cite{fazlyab2019efficient,wang2022quantitative,revay2020lipschitz,araujo2023unified,wang2023direct,havens2024exploiting,fazlyab2024certified,wang2024scalability,pauli2024novel,pauli2023lipschitz,pauli2024lipschitz}. 

\section{Complete Sets of QCs and \\ Incremental QCs for  ReLU}
\label{sec:RevisitingQCs}

\subsection{Complete Set of QCs for  ReLU}
\label{sec:CompleteQCs}

Section~\ref{sec:bgReLUQCs} introduced three classes of QCs for  ReLU.  There are other QCs for  ReLU in the literature, e.g. polytopic QCs~\cite{ebihara21CDC,ebihara21EJC}. This raises the following question: What is the largest set $\mathcal{M} \subset\Sym^{2n_v}$ such that $\Phi \in QC(\mathcal{M})$? 
 This question is addressed by our main results in this subsection.  Formally the largest set of QCs satisfied by $\Phi$ is given by 
$\left\{ M \in \Sym^{2n_v}|  \ \Phi \in QC (M) \right\}$. This set contains, by definition, all matrices $M$ such that ReLU satisfies the corresponding QC.   Theorem~\ref{thm:ReLUQCs1}
 below provides an
explicit description for this set based on certain copositivity conditions.


To provide this explicit description, first let $\Dpm^n \subset D^n$ denote the set of $n \times n$ diagonal matrices where each diagonal entry is $\pm 1$.  There are $2^{n_v}$ matrices in this set.  Next, let 
matrices $M \in \Sym^{2n_v}$ and $D\in \Dpm^{n_v}$ 
be given.  Let $M_D\in \Sym^{n_v}$ denote the following matrix:
\begin{align}
  \label{eq:MD}
     M_D:= \bmtx D \\ \frac{1}{2} (I+D) \emtx^\top 
      M
      \bmtx D \\ \frac{1}{2} (I+D) \emtx.
\end{align}
Given this notation, define the following set of matrices:
\begin{align}
\label{eq:Mc}
\Mc := \left\{ M\in \Sym^{2n_v} \, : \,
M_D \in COP^{n_v} \,\, \forall D \in \Dpm^{n_v}
\right\}.
\end{align}
The matrix dimension $2n_v$ is omitted from the notation $\Mc$ but should be clear from context.
Each matrix in $\Mc$ satisfies a collection of $2^{n_v}$ copositivity conditions; one condition for each $D\in \Dpm^{n_v}$. The next theorem states that $\Mc$ defines the complete set of QCs for the  ReLU.

\vspace{0.1in}
\begin{theorem}
\label{thm:ReLUQCs1}
    The  ReLU $\Phi$ satisfies a QC defined by $M \in \Sym^{2n_v}$ if and only if $M\in \Mc$.     
    Equivalently,
     \begin{align}         
     \Mc = \left\{ M \in \Sym^{2n_v}|  \ \Phi \in QC (M) \right\}.
      \end{align}   
\end{theorem}
\vspace{0.1in}
\begin{proof}
\noindent
($\Leftarrow$) Suppose $M\in \Mc$. Take any $v\in \R^{n_v}$ and let $w=\Phi(v)$. Define a  diagonal matrix $\bar{D} \in \Dpm^{n_v}$ and vector  $\bar{v}\in \R^{n_v}_{\ge 0}$ by\footnote{Here we use the convention that 
sign$(\cdot): \R \to \{-1,+1\}$ with
sign$(x)=+1$ if $x\ge 0$ and sign$(x)=-1$ if $x<0$.}:
\begin{align}
  \bar{D}_{kk}:=\mbox{sign}(v_k) 
  \mbox{ and }\bar{v}_k := |v_k|.
\end{align}
Thus $v=\bar{D}\bar{v}$ and $w=\frac{1}{2}(I+\bar{D}) \bar{v}$.
The equality for $w$ follows because $\frac{1}{2}(1+ \bar{D}_{kk}) \bar{v}_k$ is equal to 0 when $v_k<0$ and is equal to $v_k$ when $v_k\ge 0$.  

Let $M_{\bar D}$ denote the scaled matrix defined as in \eqref{eq:MD} but with $\bar{D} \in \Dpm^{n_v}$. Thus $M_{\bar D}$ is copositive since $M \in \Mc$ by assumption.
Combining these facts gives:
\begin{align}
\bmtx v \\ w\emtx^\top 
M
\bmtx v \\ w\emtx 
=
\bar{v}^\top M_{\bar{D}}   \bar{v} \ge 0.
\end{align}
Since $v$ was arbitrary the above holds for all $v\in \R^{n_v}$ and $w=\Phi(v)$; namely,  the  ReLU $\Phi$  satisfies the QC defined by $M \in \Mc$ as required.

\vspace{0.1in}
\noindent
($\Rightarrow$) This direction is by contrapositive.  Assume $M\notin \Mc$.  Then $M_D$ in \eqref{eq:MD} is \emph{not} copositive for some $D \in \Dpm^{n_v}$. Thus there exists $\bar{v} \in \R_{\ge 0}^{n_v}$ such that $\bar{v}^\top M_D \bar{v} <0$.  Define $v:=D \bar{v}$ and $w:=\frac{1}{2}(I+D) \bar{v}$. If $D_{kk}=-1$ then $v_k\le 0$ and $w_k=0$.  If $D_{kk}=+1$ then $v_k\ge 0$ and $w_k = v_k$. Hence $w=\Phi(v)$ and
\begin{align}
\bmtx v \\ w\emtx^\top 
M
\bmtx v \\ w\emtx 
=
\bar{v}^\top M_D \bar{v} < 0.
\end{align}
Thus  ReLU does not satisfy the QC defined by $M$.
    
\end{proof}
\vspace{0.1in}

$\Mc$ defines the complete set of QCs satisfied
by  ReLU by Theorem~\ref{thm:ReLUQCs1}. 
Section~\ref{sec:bgReLUQCs} discussed several existing QCs for  ReLU.  The next lemma and corollary provide an independent proof for these existing QCs relying on the characterization in Theorem~\ref{thm:ReLUQCs1}.

\vspace{0.1in}
\begin{lemma}
   \label{lem:ExistingQCs}
   The sets $\Mset{1}$, $\Mset{2}$, and $\Mset{3}$
   are subsets of $\Mc$.
\end{lemma}
\begin{proof}
    The proof is given in Appendix~\ref{sec:ProofOfLemExistingQCs}.
\end{proof}
\vspace{0.1in}
\begin{corollary}    
    The  ReLU $\Phi:\R^{n_v} \to \R_{\ge 0}^{n_v}$
    satisfies the QCs defined by any matrices
    in $\Mset{1}$, $\Mset{2}$, $\Mset{3}$. 
\end{corollary}
\begin{proof}
This  follows from Theorem~\ref{thm:ReLUQCs1}
and Lemma~\ref{lem:ExistingQCs}.
\end{proof}
\vspace{0.1in}

The complete set $\Mc$ contains,  by Theorem~\ref{thm:ReLUQCs1}, all other classes of QCs for  ReLU.  As discussed later in Section~\ref{sec:numimp},  numerical implementations may favor, for computational reasons, the use of subsets of QCs for  ReLU over the complete set $\Mc$.  The complete set $\Mc$ can be explored to help derive new subsets of QCs that may have computational advantages. In this sense, the complete set $\Mc$ provides a unifying view for  ReLU QCs. 



\subsection{Functions Characterized by Complete Set of QCs}

Note that given any subsets $\mathcal{P}, \, \mathcal{N}\subset\Sym^{2n_v} $ we have
$QC(\mathcal{P}\cup \mathcal{N})=QC(\mathcal{P})\cap QC(\mathcal{N})$. Thus the by making $\mathcal{M}$ as large as possible we are making the set of functions
$QC(\mathcal{M})$ as small as possible.
Theorem~\ref{thm:ReLUQCs1} then leads to a related question:  Are there other functions, besides  ReLU, that satisfy all QCs defined by matrices in $\Mc$? The answer to this question is yes. To show this, in terms of $\Phi$ define the new function $\fflip{\Phi}:\R^{n_v} \to \R^{n_v}$ by $\fflip{\Phi}(v):=-\Phi(-v)$. Now, notice 
the pair $(v,w)$ lies on the graph of $\Phi$ if and only
if $(-v,-w)$ lies on the graph of $\fflip{\Phi}$. Therefore we call
$\fflip{\Phi}$ the  flipped ReLU.  Next
note that quadratic forms are even functions:
\begin{align}
\bmtx v \\ w\emtx^\top 
M
\bmtx v \\ w\emtx 
=
\bmtx -v \\ -w\emtx^\top 
M
\bmtx -v \\ -w\emtx.
\end{align}
With this observation we can  conclude the following.

\vspace{0.1in}
\begin{proposition}
Given a subset $\mathcal{M} \subset \Sym^{2n_v}$.  Then $\Phi\in QC(\mathcal{M})$ if and only if $\fflip{\Phi}\in QC(\mathcal{M})$. 
\end{proposition}
\vspace{0.1in}
\noindent
In words, this says there is a fundamental limit on how finely QCs can specify the ReLU: the flipped ReLU must always also be included. The next theorem
states that  ReLU and flipped ReLU  are, in fact, the only functions that satisfy all QCs defined by $\Mc$. Namely, the set of QCs defined by $\Mc$ is as tight as possible for  ReLU up to the sign-invariance inherent in quadratic forms. 


\vspace{0.1in}
\begin{theorem}
\label{thm:ReLUQCs2}
    
    A function $F:\R^{n_v} \to \R^{n_v}$ 
    satisfies all QCs defined by $\Mc$ if and only if $F$ is either $\Phi$ or $\fflip{\Phi}$.  
    Namely: the set $QC(\Mc) = \{\Phi, \, \fflip{\Phi}\}$.
\end{theorem}
\begin{proof}

\noindent
($\Leftarrow$)  Assume $F\in\{\Phi, \, \fflip{\Phi}\}$. If $F= \Phi$ then Theorem~\ref{thm:ReLUQCs1} implies that $F$ satisfies all QCs defined by $\Mc$. The same is true
when $F= \fflip{\Phi}$ because quadratic forms are even functions.

\vspace{0.1in}
\noindent
($\Rightarrow$)  Suppose $F\in QC(\Mc)$. By Lemma~\ref{lem:ExistingQCs}, $\Mc$ contains
any $M_1:=\bsmtx 0 & Q_1 \\ Q_1 & -2Q_1 \esmtx$ with $Q_1$ diagonal.
Select $Q_1$ to have all zero entries except for entry $(i,i)$.
Then all $v\in \R^{n_v}$, $w=F(v)$ satisfy
\begin{align}
\label{eq:Fqc1}
    \bmtx v \\ w \emtx^\top M_1 \bmtx v \\ w \emtx 
      = 2 (Q_1)_{ii}  w_i (v_i-w_i) \ge 0.
\end{align}
$(Q_1)_{ii}$ can be chosen to be positive or negative. Thus
\eqref{eq:Fqc1} implies that $w_i=0$ or $w_i=v_i$.  In other words, {the allowable values} of $F_i(v)$ only depend on $v_i$ and {thus we write} $F_i(v_i) \in \{0, v_i\}$ for all $v_i\in\R$.

Lemma~\ref{lem:ExistingQCs} also implies that $\Mc$ contains any $M_2\in \Mset{2}$. 
Define $M_2\in \Mset{2}$ by
\begin{align}
M_2:=\bsmtx -I & I \\ 0 & I \esmtx^\top  Q_2 \bsmtx -I & I \\ 0 & I \esmtx
\mbox{ with }
Q_2 = \bsmtx 0 & 0 \\ 0 & e_i e_j^\top + e_j e_i^\top \esmtx,
\, i\ne j.
\end{align}
Then all $v\in \R^{n_v}$, $w=F(v)$ satisfy
\begin{align}
\label{eq:Fqc2}
    \bmtx v \\ w \emtx^\top M_2 \bmtx v \\ w \emtx 
      = 2 w_i w_j      \ge 0.
\end{align}
This condition is violated if $F_i(v_i)$ and $F_j(v_j)$ have opposite sign for some $v_i, v_j \in \R$.
This implies: (i) $F_i$ is globally nonnegative for all $i=1,\ldots n_v$, or (ii)
$F_i$ is globally non-positive for all $i=1,\ldots n_v$. The remainder of the proof shows that case (i) implies $F=\Phi$ while case (ii) implies $F=\fflip{\Phi}$.

First consider case (i), {and note that in order to show $F=\Phi$ it is sufficient to demonstrate  $F_i(v_i)\geq v_i$ for all $i$ and $v$.} Define $M_2\in \Mset{2}$ with
$Q_2 = \bsmtx e_i e_j^\top + e_j e_i^\top & 0 \\ 0 & 0 \esmtx$ for some $i\ne j$. Then all $v\in \R^{n_v}$, $w=F(v)$ satisfy
\begin{align}
\label{eq:Fqc3}
    \bmtx v \\ w \emtx^\top M_2 \bmtx v \\ w \emtx 
      = 2 (w_i-v_i) (w_j-v_j)      \ge 0.
\end{align}
{Now set $v_j=-1$ and then by the non-negativity of $F_j$ we have $w_j-v_j>0$.  Then (\ref{eq:Fqc3}) implies 
$w_i\geq v_i$, namely $F_i(v_i)\geq v_i$ as required.  

Case (ii) leads to $F=\fflip{\Phi}$ following a similar argument.}

\end{proof}
\vspace{0.1in}

\subsection{Complete Set of Incremental QCs for  ReLU}
\label{sec:CompleteIncQCs}

Next, we build on our results 
in Section~\ref{sec:CompleteQCs}
to define the largest class of incremental QCs for  ReLU.
Let matrices $M \in \Sym^{2n_v}$ and $D_1$, 
 $D_2\in \Dpm^{n_v}$  be given.  Let $M_{D_1,D_2}\in \Sym^{2n_v}$ denote the following matrix:
\begin{align}
  \label{eq:MD1D2}
     M_{D_1D_2}:= 
      \bsmtx D_1 & -D_2 \\ \frac{1}{2} (I+D_1) 
      & -\frac{1}{2} (I+D_2) \esmtx^\top
      M
      \bsmtx D_1 & -D_2 \\ \frac{1}{2} (I+D_1) 
      & -\frac{1}{2} (I+D_2) \esmtx.
\end{align}
Given this notation, define the following set of matrices:
\begin{align}
\label{eq:Mi}
\Mic & := \left\{ M\in \Sym^{2n_v} \, : \, 
 M_{D_1D_2} \in COP^{2n_v} \,\, \forall 
D_1,\, D_2 \in \Dpm^{n_v}
\right\}.
\end{align}
Note that $\Mic$ involves a collection $2^{2n_v}=4^{n_v}$ copositivity
conditions; one condition for each pair $D_1$, $D_2\in D^{n_v}_{\pm1}$.  The next theorem states that $\Mic$ defines the complete class of incremental QCs for  ReLU.

\vspace{0.1in}
\begin{theorem}
\label{thm:ReLUIncQCs}
    The  ReLU $\Phi$ satisfies an incremental QC defined by $M \in \Sym^{2n_v}$ if and only if $M\in \Mic$.     Equivalently,
     \begin{align}         
     \Mic = \left\{ M \in \Sym^{2n_v}|  \ \Phi \in QC^{inc}(M) \right\}.
      \end{align}   
\end{theorem}
\vspace{0.05in}
\begin{proof}

The definition of an incremental QC involves
a quadratic form \eqref{eq:IncQCdef} that can be equivalently written as:
\begin{align*}
& \bmtx \bar{v}-\hat{v} \\ \bar{w}-\hat{w} \emtx^\top 
M
\bmtx \bar{v}-\hat{v} \\ \bar{w}-\hat{w} \emtx
=
\bmtx \bar{v} \\ \hat{v} \\ \bar{w} \\ \hat{w} \emtx^\top
R^\top M R
\bmtx \bar{v} \\ \hat{v} \\ \bar{w} \\ \hat{w} \emtx \\
& \mbox{ where }  R := \bmtx I & -I & 0 & 0 \\ 0 & 0 & I & - I \emtx \in \R^{2n_v \times 4n_v}.
\end{align*}
Thus $M$ defines an incremental constraint for a  ReLU of dimension $n_v$ if and only if
$R^\top M R$ defines a normal (non-incremental) QC
for a  ReLU of dimension $2n_v$. By Theorem~\ref{thm:ReLUQCs1}, this is equivalent to the following condition:
\begin{align}
\label{eq:IncQC1}
    \bmtx D \\ \frac{1}{2} (I+D) \emtx^\top 
    (R^\top  M R)
    \bmtx D \\ \frac{1}{2} (I+D) \emtx
    \in COP^{2n_v} 
    \,\, \forall D \in D^{2n_v}_{\pm 1}.
\end{align}
Block partition $D=\bsmtx D_1 & 0 \\ 0 & D_2\esmtx$.
Then \eqref{eq:IncQC1} simplifies, by direct multiplication, to $M_{D_1D_2}  \in COP^{2n_v}$
for all  $D_1$, $D_2 \in D^{n_v}_{\pm 1}$. 
In summary, $M$ defines an incremental QC for 
ReLU if and only if   $M\in \Mic$.    
\end{proof}
\vspace{0.1in}

\subsection{New Incremental QCs for ReLU}

Section~\ref{sec:bgReLUIncQCs} introduced a well-known set of ReLU incremental QCs, denoted $\Miset{1}$,  based on $[0,1]$ slope constraint for scalar ReLU.  We initially conjectured that  $\Miset{1}$ might contain all possible incremental QCs for ReLU, i.e $\Miset{1} =\Mic$. However, we found examples where the use of $\Miset{1}$ yields more conservative bounds than $\Mic$ (See Section~\ref{sec:LipschitzEx}).  Such examples imply that $\Miset{1}$ is a strict subset of $\Mic$. 

The purpose of this subsection is to derive a new set of ReLU incremental QCs that we found while studying these examples. This new set is interesting on its own. Moreover, this illustrates that $\Mic$ can be explored to help derive new sets of QCs that may have computational advantages. Our new set of incremental QCs is based on the following property for scalar ReLU.
\vspace{0.1in}
\begin{proposition}
\label{prop:ScalarIncQC}
Any $\bar{v}_i$, $\hat{v}_i\in \R$  $(i=1,2)$ with $\bar{w}_i=\phi(\bar{v}_i)$ and $\hat{w}_i=\phi(\hat{v}_i)$ satisfy:
\begin{align}
\label{eq:ScalarIncQC}
\left[
(\bar{v}_1-\hat{v}_1) - (\bar{v}_2-\hat{v}_2)
\right]^2
 +2(\bar{w}_1-\hat{w}_1) \cdot
(\bar{w}_2-\hat{w}_2 ) \ge 0.
\end{align}
\end{proposition}
\vspace{0.05in}
\begin{proof}
To simplify notation, let
$dv_i:=\bar{v}_i-\hat{v}_i$ and 
$dw_i:=\bar{w}_i-\hat{w}_i$. With this notation, we need to show:
\begin{align}
\label{eq:ScalarIncQCSimple}
\left[ dv_1 - dv_2 \right]^2
 +2 dw_1 \cdot dw_2 \ge 0.
\end{align}
The $[0,1]$ slope constraint implies that if $dv_i\ge 0$ then  $dw_i \in [0,dv_i]$. Similarly, if $dv_i\le 0$ then $dw_i \in [dv_i,0]$.  


If $dv_1 dv_2\ge 0$ then $dw_1 dw_2\ge 0$ and thus 
\eqref{eq:ScalarIncQCSimple} holds.  On the other hand, if $dv_1 dv_2 \le 0$ then  $dw_1 dw_2 \ge dv_1 dv_2$. This case implies  $(dv_1-dv_2)^2+2dw_1 dw_2 \ge dv_1^2 + dv_2^2 \ge 0$.
\end{proof}
\vspace{0.1in}

Having established this new property, we now state the following lemma about doubly hyperdominant matrices.

\vspace{0.1in}
\begin{lemma}
\label{lem:DHdecomp}
Let $T_2=T_2^\top \in DH^{n_v}$ be given
with $\sum_{k=1}^{n_v} T_{ik}=0$ for $i=1,\ldots,n_v$.    Then there exists
$\{ \lambda_{ij} \}_{i,j=1}^{n_v} \in \R_{\ge 0}$ such that
\begin{align}
   \label{eq:T2decomp}
   T_2 = \sum_{i,j=1}^{n_v} \lambda_{i,j} (e_i-e_j)(e_i-e_j)^\top.
\end{align}
\end{lemma}
\vspace{0.05in}
\begin{proof}
The rows $T_2$ of sum to zero and hence the columns also sum to zero by symmetry. Hence, $T_2$ is a symmetric, doubly hyperdominant matrix with zero excess.  Theorem 3.7 in \cite{Willems:71} gives a decomposition similar to \eqref{eq:T2decomp}
for nonsymmetric, doubly hyperdominant matrices with zero excess.  The argument in \cite{Willems:71} can be adapted as follows for the symmetric case. The proof is given Appendix~\ref{sec:ProofOfDHdecomp}.
\end{proof}
\vspace{0.1in}

Next we state a new set of incremental QCs for  ReLU using the decomposition
in Lemma~\ref{lem:DHdecomp} and the property for scalar ReLU in Proposition~\ref{prop:ScalarIncQC}.

\vspace{0.1in}
\begin{theorem}
    \label{thm:IncQC}
    The  ReLU $\Phi:\R^{n_v} \to \R_{\ge 0}^{n_v}$
    satisfies the incremental QC defined by any matrix in the
    following set:
    \begin{align}
    \label{eq:Mi2}
    \Miset{2} & := \left\{ \bmtx T_2 & 0 \\ 0 & S_2-T_2 \emtx
    \, : \,  T_2=T_2^\top \in DH^{n_v}, \, \right. \\
    \nonumber
    & \left. 
    S_2\in D^{n_v} ,\  (S_2)_{ii} = (T_2)_{ii}, \,
    \sum_{k=1}^{n_v} T_{ik}=0 \mbox{ for } i=1,\ldots,n_v \right\}.    
 \end{align}
\end{theorem}
\vspace{0.05in}
\begin{proof}
 Take any $M_2\in \Miset{2}$. The upper left block $T_2$ is a symmetric, doubly hyperdominant matrix with zero excess. By Lemma~\ref{lem:DHdecomp}, there exists 
 $\{ \lambda_{ij} \}_{i,j=1}^{n_v} \in \R_{\ge 0}$ such that
\begin{align}
   T_2 = \sum_{i,j=1}^{n_v} \lambda_{i,j} (e_i-e_j)(e_i-e_j)^\top.
\end{align}
This further implies that $S_2=\sum_{i,j=1}^{n_v} \lambda_{ij} ( e_i e_i^\top + e_j e_j^\top)$ and hence $S_2-T_2 = \sum_{i,j=1}^{n_v} \lambda_{ij} ( e_ie_j^\top + e_je_i^\top)$.   Therefore
$M_{2}$ can be decomposed as:
\begin{align}
  \label{eq:Mi2decomp}
  M_{2} = \sum_{i,j=1}^{n_v}  \lambda_{ij} 
  \bsmtx (e_i-e_j)(e_i-e_j)^\top & 0 \\
             0 & e_ie_j^\top + e_j e_i^\top \esmtx.
\end{align}
Then for any
$\bar{v}, \, \hat{v}\in \R^{n_v}$ and
$\bar{w}=\Phi(\bar{v})$, $\hat{w}=\Phi(\hat{v})$,
\begin{align*}
& \bmtx \bar{v}-\hat{v} \\ \bar{w}-\hat{w} \emtx^\top 
M_{2}
\bmtx \bar{v}-\hat{v} \\ \bar{w}-\hat{w} \emtx \\
&
= \sum_{i,j=1}^{n_v}  \lambda_{ij}
   \left[ \left( (\bar{v}_i-\hat{v}_i)
        - (\bar{v}_j-\hat{v}_j)
   \right)^2
   + 2 (\bar{w}_i-\hat{w}_i)
        \cdot (\bar{w}_j-\hat{w}_j)
    \right].
\end{align*}
Each term of this sum is nonnegative by property \eqref{eq:ScalarIncQC}
 of the scalar ReLU and $\lambda_{ij}\ge 0$.
Thus $\Phi$ satisfies all incremental QCs defined by $\Miset{2}$. 
\end{proof}
\vspace{0.1in}

ReLU satisfies the incremental QCs defined by  $\Miset{1}$ and $\Miset{2}$.  Numerical implementations may favor these subsets over the complete set $\Mic$ for computational reasons.


\section{Related Extensions}
\label{sec:relatedresults}

The key results in the previous sections can be generalized in various ways.  This section presents two specific extensions leading to the complete set of QCs for other nonlinear functions appearing in the literature.

\subsection{Affine Transformations}

First we consider the effect of affine transformations on QCs. The next result is stated for general functions and is not specific to  ReLU.

\vspace{0.1in}
\begin{lemma}    
\label{lem:AffineQC}
Let a function $F:\R^{n_v}\to \R^{n_w}$ and matrices $A_0\in \R^{n_w\times n_v}$, $A_1\in \R^{n_w\times n_w}$, and $A_2\in \R^{n_v\times n_w}$ be given. Define a new function $G:\R^{n_v}\to\R^{n_w}$ by:
\begin{align}
\label{eq:Gdef}
G(\hat{v}) = A_0 \hat{v}  
+ A_1 F( A_2 \hat{v} ).    
\end{align}
\begin{enumerate}
\item[(a)]  Given $M \in \Sym^{n_v+n_w}$.  If $A_1$ is nonsingular and $F$ satisfies the QC defined by $M$  then $G$ satisfies the QC defined by $\hat{M}:= R_A^\top M R_A\in \Sym^{n_v+n_w}$ with
\begin{align}
\label{eq:RAdef}
    R_A:= \bmtx A_2 & 0 \\ -A_1^{-1}A_0 & A_1^{-1} \emtx.
\end{align}

\item[(b)] Given $\hat{M} \in \Sym^{n_v+n_w}$. If $A_1$ and $A_2$ are nonsingular then $G$ satisfies the QC defined by $\hat{M}$ if and only if $F$ satisfies the QC defined by $M:= R_A^{-\top} \hat{M} R_A^{-1}$.
\end{enumerate}
\end{lemma}
\vspace{0.05in}
\begin{proof}

\noindent
(a) $F$ satisfies the QC defined by $M$ and hence,
\begin{align}
\label{eq:FQCforAffine}
\bmtx A_2 \hat{v} \\ F( A_2 \hat{v} )\emtx^\top
M
\bmtx A_2 \hat{v} \\ F( A_2 \hat{v} )\emtx
\ge 0
\,\,\,\, \forall \hat{v} \in \R^{n_v}.
\end{align}
Next, rearranging \eqref{eq:Gdef} we get
\begin{align}
    F(A_2\hat{v}) = A_1^{-1} \left(
     G(\hat{v})- A_0 \hat{v} \right).
\end{align}
Substitute this expression into \eqref{eq:FQCforAffine} to obtain
\begin{align}
\label{eq:GQCforAffine}
\bmtx \hat{v} \\ G( \hat{v} )\emtx^\top
R_A^\top  M R_A
\bmtx \hat{v} \\ G(\hat{v} )\emtx
\ge 0
\,\,\,\, \forall \hat{v} \in \R^{n_v}.
\end{align}
Hence $G$ satisfies the QC defined by $\hat{M}$ as required.

\vspace{0.1in}
\noindent
(b)  First, note that if $A_1$ and $A_2$ are nonsingular then $R_A$ is nonsingular with the following inverse:
\begin{align}
    R_A^{-1} = \bmtx A_2^{-1} & 0 \\
        A_0 A_2^{-1}  & A_1 \emtx.
\end{align}
It follows from (a) that if $F$ satisfies the QC defined by $M=R_A^{-\top} \hat{M} R_A^{-1}$ then $G$ satisfies the QC defined by $R_A^\top M R_A=\hat{M}$.  Thus it remains to show the ``only if" direction. Assume $G$ satisfies the QC defined by $\hat{M}$. Since $A_2$ is nonsingular, this implies:
\begin{align}
\label{eq:GQCforAffine2}
\bmtx A_2^{-1} v \\ G( A_2^{-1} v )\emtx^\top
\hat{M}
\bmtx A_2^{-1} v \\ G( A_2^{-1} v )\emtx
\ge 0
\,\,\,\, \forall v \in \R^{n_v}.
\end{align}
Substitute for $G(A_2^{-1} v)$ using the definition of $G$ in \eqref{eq:Gdef}. Then
the QC \eqref{eq:GQCforAffine2} simplifies to:
\begin{align}
\bmtx v \\ F( v )\emtx^\top
R_A^{-\top} \hat{M} R_A^{-1}
\bmtx v \\ F( v )\emtx
\ge 0
\,\,\,\, \forall v \in \R^{n_v}.
\end{align}
Hence $F$ satisfies the QC defined by $M
=R_A^{-\top} \hat{M} R_A^{-1}$.
\end{proof}
\vspace{0.1in}

We can use affine transformations to give the complete set of QCs for another class of  functions. Define $g_{\alpha\beta}: \R \to \R$ for $\alpha\ne \beta$ as follows:
\begin{align}
g_{\alpha\beta}(v)= \left\{
  \begin{array}{ll}
    \alpha v & \mbox{if } v < 0 \\
    \beta v & \mbox{if } v \geq 0 
  \end{array} 
  \right. .
\end{align}
$g_{\alpha\beta}$ corresponds to the scalar ReLU $\phi$ when $\alpha=0$ and $\beta=1$.  It corresponds to leaky ReLU \cite{maas2013rectifier} when $0<\alpha<1$
and $\beta=1$.  In general, $g_{\alpha\beta}$ is a piecewise linear function with a slope change at $v=0$.
The corresponding repeated function  
$G_{\alpha\beta}:\R^{n_v} \to \R^{n_v}$  is defined elementwise by
$G_{\alpha\beta}(v)= (g_{\alpha\beta}(v_1), \ldots , \, g_{\alpha\beta}(v_{n_v}))$.

Corollary~\ref{cor:QCsforGab} below states that the complete
set of QCs for $G_{\alpha\beta}$ is defined
by the following set of matrices:

{\small
\begin{align}
\Mset{\alpha\beta} & := \left\{ \hat{M}\in \Sym^{2n_v} \, : \, \right. \\
\nonumber
& \left.
\bsmtx D \\ \alpha D + \frac{\beta-\alpha}{2}(I+D) \esmtx^\top 
\hat{M}
\bsmtx D \\ \alpha D + \frac{\beta-\alpha}{2}(I+D) \esmtx \in COP^{n_v}  
\,\, \forall D \in \Dpm^{n_v}
\right\}.
\end{align}
}

\vspace{0.1in}
\begin{corollary}    
\label{cor:QCsforGab}
    The function $G_{\alpha\beta}:\R^{n_v} \to \R^{n_v}$ with $\alpha \ne \beta$
    satisfies the QC defined by $\hat{M}$ if and only if $\hat{M}\in \Mset{\alpha\beta}$. 
\end{corollary}
\begin{proof} $G_{\alpha\beta}$ can be written in terms of the  ReLU as follows:
\begin{align}
G_{\alpha\beta}(v) =  \alpha v + (\beta-\alpha) \Phi(v).
\end{align}
This corresponds to an affine transformation
as in \eqref{eq:Gdef} with $A_0 = \alpha I$, $A_1 = (\beta-\alpha) I$, and $A_2 = I$.  Both $A_1$ and $A_2$ are nonsingular and thus we have:
\begin{align}
R_A = \bmtx I & 0 \\ -\frac{\alpha}{\beta-\alpha} I &  \frac{1}{\beta-\alpha} I \emtx \mbox{ and } 
R_A^{-1}  = \bmtx I & 0 \\  \alpha I &  (\beta-\alpha) I \emtx.     
\end{align}

It follows from Lemma~\ref{lem:AffineQC}(b) that $G$ satisfies the QC defined by $\hat{M}$ if and only if 
 ReLU $\Phi$ satisfies the QC defined by $M=R_A^{-\top} \hat{M} R_A^{-1}$. 
Moreover, Theorem~\ref{thm:ReLUQCs1} states that $\Phi$ satisfies a QC defined by $M$ if and only if $M \in \Mc$. Combining these equivalences, $G$  satisfies the QC defined by $\hat{M}$ if and only if
\begin{align}
\label{eq:MabCond}
\bmtx D \\ \frac{1}{2} (I+D) \emtx^\top 
R_A^{-\top} \hat{M} R_A^{-1}
\bmtx D \\ \frac{1}{2} (I+D) \emtx \in COP^{n_v}
\,\, \forall D \in \Dpm^{n_v}.
\end{align}
This simplifies to $\hat{M}\in \Mset{\alpha\beta}$.
\end{proof}
\vspace{0.1in}

\subsection{Application to Householder and Max/Min Activations}

Several gradient norm preserving activation functions have appeared in the literature
for the design of Lipschitz neural networks 
\cite{tanielian2021approximating,anil2019sorting}. One example is the Householder activation function \cite{singla2022improved}.
Given a vector $h\in \R^{n_v}$ with $\|h\|_2=1$, the Householder activation $G_h:\R^{n_v}\to \R^{n_v}$ is defined by:
\begin{align}
\label{eq:Householder}
G_h(v)= \left\{
  \begin{array}{ll}
    v & \mbox{if } h^\top v \ge 0 \\
    (I-2hh^\top) v & \mbox{if } h^\top v < 0 
  \end{array} 
  \right. .
\end{align}
As one example, if $h=\frac{1}{\sqrt{2}} \bsmtx 1 & -1 \esmtx^\top$ then  $G_h(v) = \bsmtx \max(v_1,v_2) \\ \min(v_1,v_2) \esmtx$. This special case of the Householder activation is called the MaxMin activation \cite{tanielian2021approximating,anil2019sorting}.   A more general "groupwise" Householder and MaxMin activation functions are used in~\cite{tanielian2021approximating,anil2019sorting,singla2022improved} but we will focus on the (single group) Householder in \eqref{eq:Householder} for simplicity.

The Householder activation function can be expressed with an affine transformation on  ReLU.  Let $1_{n_v}\in \R^{n_v}$ be the vector of all ones and $N\in \R^{n_v \times n_v}$ be any matrix with $N 1_{n_v}=0$. Then $G_h$ can be written as:
\begin{align}
\label{eq:HouseholderAffine}
    G_h(x) = v + \left(\frac{2}{n_v} h 1_{n_v}^\top + N \right) \,
    \Phi\left( -(h^\top v) \cdot 1_{n_v} \right).
\end{align}
Equation~\ref{eq:HouseholderAffine} corresponds to an affine transformation as in \eqref{eq:Gdef} with $A_0 = I$, $A_1 = \frac{2}{n_v} h 1_{n_v}^\top + N$, and $A_2 = -1_{n_v}h^\top$. This transformation is not unique due to the choice of $N$. A similar affine transformation appeared in \cite{pauli2024novel} for the MaxMin activation.

It is possible to choose $N$ such that $A_1$ is nonsingular.  By Lemma~\ref{lem:AffineQC}(a),  if  ReLU $\Phi$ satisfies a QC defined by $M$ then the Householder activation $G_h$ satisfies the QC defined by $\hat{M}:= R_A^\top M R_A$
with $R_A$ defined in \eqref{eq:RAdef}. Thus we can use the complete set of QCs for  ReLU to define QCs for $G_h$. However, $A_2$ is singular in this affine transformation. Hence Lemma~\ref{lem:AffineQC}(b) does not apply, i.e. we cannot equivalently map between QCs for $\Phi$ and $G_h$.  Therefore, the complete set of QCs for $G_h$ are not necessarily constructed from the complete set for $\Phi$.


We can provide a direct construction for the complete set of QCs of the Householder activation $G_h$.  This uses a similar technique as used to define $\Mc$ for  ReLU. First, define the following set of matrices:
\begin{align}
\label{eq:Mh}
\Mset{h} & := \left\{ \hat{M}\in \Sym^{2n_v} \, : \, \right. \\
\nonumber
& \hspace{0.3in} \left.
\bsmtx I \\ I + (d-1) h h^\top \esmtx^\top 
\hat{M}
\bsmtx I \\ I + (d-1) h h^\top \esmtx
\succeq 0
\,\, \mbox{ for } d=\pm 1.
\right\}.
\end{align}

Note that each matrix in $\Mset{h}$ satisfies only two positive semidefiniteness conditions; no copositivity conditions arise in this case.  The next theorem states that $\Mset{h}$ defines the complete set of QCs for the Householder activation $G_h$.

\vspace{0.1in}
\begin{theorem}
\label{thm:HouseholderQCs}
    Let $h\in \R^{n_v}$ be given with $\|h\|_2=1$ and suppose $\hat{M} \in \Sym^{2n_v}$. The Householder activation $G_h\in {QC}(\hat{M})$  if and only if $\hat{M} \in \Mset{h}$.
\end{theorem}
\begin{proof}

\noindent
($\Leftarrow$) Suppose $\hat{M}\in \Mset{h}$. Take any $v\in \R^{n_v}$ and let $w=G_h(v)$.
Define $d=$sign$(h^\top v)$ so that  $w = ( I + (d-1) h h^\top ) v$. This gives:
\begin{align*}
\bmtx v \\ w \emtx^\top 
\hat{M}
\bmtx v \\ w \emtx 
=
v^\top 
\bmtx I \\ I + (d-1) h h^\top \emtx^\top 
\hat{M}
\bmtx I \\ I + (d-1) h h^\top \emtx
v. 
\end{align*}
The expression above is nonnegative since $\hat{M}\in\Mset{h}$. 
Since $v$ was arbitrary the above holds for all $v\in \R^{n_v}$ and $w=G_h(v)$. Hence the Householder activation $G_h$  satisfies the QC defined by $M \in \Mset{h}$ as required.

\vspace{0.1in}
\noindent
($\Rightarrow$) This direction is by contrapositive.  Assume $\hat{M}\notin \Mset{h}$.   There exists $d=+ 1$ or $-1$ and a vector $v\in \R^{n_v}$ such that:
\begin{align}
\label{eq:MhFails}
v^\top \bmtx I \\ I + (d-1) h h^\top \emtx^\top 
\hat{M}
\bmtx I \\ I + (d-1) h h^\top \emtx
v < 0.
\end{align}
Equation~\ref{eq:MhFails} holds if we use either $+ v$ or $-v$. Select the sign of $v$ so that $d=$sign$( h^\top v)$.\footnote{If $h^\top v=0$ then the term $(d-1)hh^\top v$ in
\eqref{eq:MhFails} is zero.  Hence \eqref{eq:MhFails} holds for both $d=+1$ and $d=-1$. In this case we can assume $d=+1$ without loss of generality.}  As a consequence,
\begin{align*}
  (I+(d-1)hh^\top) v = 
  \left\{
  \begin{array}{ll}
    v & \mbox{if } h^\top v \ge 0 \\
    (I-2hh^\top) v & \mbox{if } h^\top v < 0 
  \end{array} 
  \right. .
\end{align*}
Hence this term is equal to $G_h(v)$ and Equation~\ref{eq:MhFails} can be written as:
\begin{align}
\bmtx v \\ G_h(v) \emtx^\top 
\hat{M}
\bmtx v \\ G_h(v) \emtx 
< 0.
\end{align}
Thus the Householder activation does not satisfy the QC defined by $\hat{M}$.
\end{proof}
\vspace{0.1in}

Theorem~\ref{thm:HouseholderQCs} states that
$\Mset{h}$ defines the complete set of QCs satisfied by the Householder activation. This
also gives the complete set of QCs for the MaxMin activation as a special case when $h=\frac{1}{\sqrt{2}}\bsmtx 1 & -1 \esmtx^\top$.
We conjecture that a similar method can be used to construct the complete set of QCs for the more general "groupwise" versions of the Householder and MaxMin activations in \cite{tanielian2021approximating,anil2019sorting,singla2022improved,pauli2024novel}.

\section{Stability Analysis With ReLU QCs}
\label{sec:QCstab}

This section illustrates the use of the complete set of  ReLU QCs for stability analysis.  The stability analysis conditions are standard but the section describes numerical issues specific to this complete set. This combines our new results with a number of ideas from the literature and thus provides tutorial value.

\subsection{Dynamic Systems With ReLU Activation Functions}

Consider the interconnection shown in
Figure~\ref{fig:LFTdiagram}  with a  ReLU $\Phi$ wrapped in feedback around the top channels of a nominal system $G$.  This interconnection is denoted as $F_U(G,\Phi)$.  The nominal part $G$ is a discrete-time, linear time-invariant (LTI) system described by the following state-space model:
\begin{align}
  \label{eq:LTInom}
  \begin{split}
    & x(k+1) = A\, x(k) + B_1\,w(k) +  B_2\, d(k) \\
    & v(k)=C_{1}\,x(k)+D_{11}\, w(k)+ D_{12} \,d(k)\\
    & e(k)=C_{2}\,x(k)+D_{21}\, w(k)+ D_{22}\,d(k),
  \end{split}
\end{align}
where $x \in \R^{n_x}$ is the state. The inputs are $w\in \R^{n_w}$ and $d\in \R^{n_d}$ while $v\in \R^{n_v}$ and $e\in \R^{n_e}$ are outputs.   The interconnection $F_U(G,\Phi)$ is known as a linear fractional transformation (LFT) in the robust control literature \cite{zhou96}. The interconnection has its roots in the Lurye decomposition used in the absolute stability problem \cite{khalil01}. Recurrent Neural Networks (RNNs) with ReLU activation functions can be decomposed into this form.

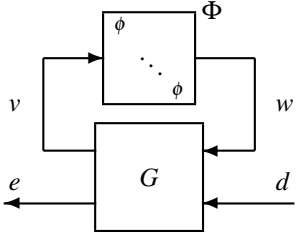
\begin{figure}[h!t]
\centering
\begin{picture}(110,90)(40,20)
 \thicklines
 \put(75,25){\framebox(40,40){$G$}}
 \put(143,40){$d$}
 \put(150,35){\vector(-1,0){35}}  
 \put(42,40){$e$}
 \put(75,35){\vector(-1,0){35}}  
    \put(78,73){\framebox(34,34){
    $\begin{smallmatrix}
    \phi & &  \\ & \,\, \ddots  & \\ 
    & & \phi 
    \end{smallmatrix}$
    }}
 \put(115,105){$\Phi$}
 \put(42,70){$v$}
 \put(55,55){\line(1,0){20}}  
 \put(55,55){\line(0,1){35}}  
 \put(55,90){\vector(1,0){23}}  
 \put(143,70){$w$}
 \put(135,90){\line(-1,0){23}}  
 \put(135,55){\line(0,1){35}}  
 \put(135,55){\vector(-1,0){20}}  
\end{picture}
\caption{Interconnection $F_U(G,\Phi)$ of a nominal discrete-time LTI system $G$ and  ReLU $\Phi$. }
\label{fig:LFTdiagram}
\end{figure}

This feedback interconnection involves an implicit equation if $D_{11}\ne 0$.  Specifically, the second equation in \eqref{eq:LTInom} combined with $w(k)=\Phi( v(k) )$ yields:
\begin{align}
   \label{eq:WellPosed}
   v(k)=C_{1}\,x(k)+D_{11}\, \Phi( v(k) )+ D_{12} \,d(k).
\end{align}
This equation is \emph{well-posed} if there exists a unique solution $v(k)$ for all  values of $x(k)$ and $d(k)$.  Well-posedness of this equation implies that the dynamic system $F_U(G,\Phi)$ is well-posed in the following sense:  for all
initial conditions $x(0)\in\R^{n_x}$ and  inputs $d\in \ell_2^{n_d}$ there exists unique solutions $x$, $v$, $w$ and $e$ in $\ell_2$ to the system $F_U(G,\Phi)$. There are simple sufficient conditions for well-posedness of \eqref{eq:WellPosed}, e.g. Lemma 1 in \cite{richardson23} (which relies on results in \cite{valmorbida18,zaccarian02}). Thus, we'll assume well-posed for simplicity in the remainder of the paper.

A well-posed interconnection $F_U(G,\Phi)$ is \emph{internally stable} if $x(k)\to 0$ from any initial condition $x(0)$ with $d(k)=0$ for $k\in \N$. In other words, $F_U(G,\Phi)$ is internally stable if $x=0$ is a globally asymptotically stable equilibrium point with no external input.  A well-posed interconnection $F_U(G,\Phi)$ has \emph{finite induced-$\ell_2$ gain} if there exists  $\gamma<\infty$ such that the output $e$ generated by any $d\in \ell_2^{n_d}$ with $x(0)=0$ satisfies $\|e\|_2 \le \gamma \, \|d\|_2$.  We denote the infimum of all such $\gamma$ by $\|F_U(G,\Phi)\|_{2 \to 2}$.


\subsection{Stability and Performance Condition}

We next state a stability and performance condition for $F_U(G,\Phi)$ using the complete set of ReLU QCs defined in Section~\ref{sec:CompleteQCs}.

\vspace{0.1in}
\begin{theorem}    
    \label{thm:StabPerfConf}
Consider the system $F_U(G,\Phi)$ 
with the LTI system $G$ defined in \eqref{eq:LTInom}
and $\Phi:\R^{n_v}\to \R_{\ge 0}^{n_v}$ a  ReLU. Assume $F_U(G,\Phi)$ is well-posed. Also assume there exists a $2n_v\times 2n_v$ matrix $M\in \Mc$, scalar $\gamma > 0$,
and $P\in \Sym^{n_x}$ with $P\succeq 0$ such that $LMI(P,M,\gamma^2)\prec 0$ where:
\begin{align}
\nonumber
& LMI(P,M,\gamma^2) := 
\bmtx A^\top P A-P  & A^\top P B_1 
  &  A^\top P B_2 \\ 
  B_1^\top P A & B_1^\top PB_1 
  & B_1^\top P B_2 \\
  B_2^\top P A & B_2^\top P B_1 
& B_2^\top P B_2-\gamma^2 I\emtx \\
\label{eq:LMI}
& \hspace{0.3in}
+ \bmtx C_2^\top \\ D_{21}^\top \\ D_{22}^\top\emtx
  \bmtx C_2^\top \\ D_{21}^\top \\ D_{22}^\top \emtx^\top
+  \bmtx C_1^\top & 0 \\ D_{11}^\top & I  \\ D_{12}^\top & 0 \emtx
M
  \bmtx C_1^\top & 0 \\ D_{11}^\top & I  \\ D_{12}^\top & 0 \emtx^\top.
\end{align}
Then $F_U(G,\Phi)$ is internally stable and
$\|F_U(G,\Phi)\|_{2\to 2} < \gamma$.
\end{theorem}    
\begin{proof}
This theorem is a standard dissipation result~\cite{schaft99,
willems72a, willems72b, khalil01,hu17}. A proof, similar to the one given in \cite{noori24ReLURNN}, is given for completeness.

The LMI is strictly feasible by assumption and hence it remains feasible under small perturbations: $LMI(P+\epsilon I,M,\gamma^2) + \epsilon I\prec 0$  for some sufficiently small $\epsilon >0$. The
system $F_U(G,\Phi)$ is well-posed by assumption. Hence it has a unique causal solution $(x, w, v, e)$ (all in $\ell_2$) for any given initial condition $x(0)$ and input $d \in \ell_2^{n_d}$.   

Define a storage function by $V\left(x\right) := x^\top(P+\epsilon I) x$.  Left and right
multiply the perturbed LMI by $[x(k)^\top, \, w(k)^\top, \, d(k)^\top]$ and its transpose.  The result, applying the  dynamics~\eqref{eq:LTInom}, gives the following condition:
\begin{align*}
   & V\left(x(k+1)\right)  - V\left(x(k)\right) 
    -\gamma^2  d(k)^\top d(k)
   +  e(k)^\top e(k)  
    \\
& + \bmtx v(k) \\ w(k) \emtx^\top
    M \bmtx v(k) \\ w(k) \emtx 
    + \epsilon \bsmtx x(k) \\ w(k) \\d(k) \esmtx^\top \bsmtx x(k) \\ w(k) \\d(k) \esmtx
    \le 0.
\end{align*}
$M\in \Mc$ defines a valid QC for  ReLU by
Theorem~\ref{thm:ReLUQCs1}. Hence the term involving $M$ is non-negative.  Thus the dissipation inequality simplifies to:
\begin{align}
\label{eq:DI}
\begin{split}
   & V\left(x(k+1)\right)-V\left(x(k)\right) +  e(k)^\top e(k)  \\
   & \le  (\gamma^2-\epsilon) d(k)^\top d(k)  - \epsilon x(k)^\top x(k).
\end{split}  
\end{align}
Internal stability and the $\ell_2$ gain bound for $F_U(G,\Phi)$ follow from this inequality.  Specifically, if $d(k)=0$ for all $k$ then \eqref{eq:DI} implies the following Lyapunov inequality:
\begin{align*}
    V(x(k+1)) - V(x(k)) \le -\epsilon x(k)^\top x(k).
\end{align*}
Hence $V$ is a Lyapunov function and the system is globally asymptotically stable (Theorem 27 in Section 5.9 of \cite{vidyasagar02}).

Next, assume $x(0)=0$ and $d \in \ell_2^{n_d}$. Summing \eqref{eq:DI} from $k=0$ to $k=T-1$ and using $V(x(0))=0$ yields:
\begin{align*}
   V\left(x(T)\right) +  \sum_{k=0}^{T-1} e(k)^\top e(k)  \le \sum_{k=0}^{T-1}  (\gamma^2-\epsilon) d(k)^\top d(k). 
\end{align*}
Note that $V(x(T))\ge 0$ because $P$ is positive semidefinite. Moreover, the right side is upper bounded by $(\gamma^2-\epsilon)\|d\|_2^2$ for all $T\in\N$.  This implies
that $e \in \ell_2$ and $\|e\|_2 < \gamma\|d\|_2$. 
\end{proof}
\vspace{0.1in}

\subsection{Numerical Implementation Using SDP Relaxation of  $\Mc$}
\label{sec:numimp}

This section focuses on the numerical issues associated with the complete set $\Mc$. This complete set is convex but checking copositivity is co-NP complete \cite{murty1987some}.  Hence we rely on a standard copositivity relaxation to formulate our stability/performance condition as an SDP. 


Let $\{D_1,\ldots,D_{2^{n_v}} \}$ denote the $2^{n_v}$ entries of $\Dpm^{n_v}$. Define the following convex optimization based on 
Theorem~\ref{thm:StabPerfConf}:
\begin{align}
   \label{eq:Optim1}
   \begin{split}       
   &  \min_{P=P^\top,\, M=M^\top, \, \gamma^2} \gamma^2  \\
   &  P\succeq 0, \,\,\, LMI(P,M,\gamma^2) \prec 0,  \\
   & \bmtx D_i\\ \frac{1}{2} (I+D_i) \emtx^\top 
      M
      \bmtx D_i \\ \frac{1}{2} (I+D_i) \emtx \in COP^{n_v},
      \,\, i=1,\ldots 2^{n_v}.
   \end{split}
\end{align}
If the optimization is feasible then $F_U(G,\Phi)$ is stable. Moreover, the optimization returns the tightest (smallest) upper bound on the $\ell_2$ gain (using a quadratic storage and the complete set of ReLU QCs). The set of copositive matrices is a closed, convex cone (Proposition 1.24 in \cite{berman2003completely}).  It follows that $\Mc$ is a convex cone and \eqref{eq:Optim1} is a convex optimization. 

However,  simply testing if a matrix is copositive is an co-NP complete problem \cite{murty1987some}.  Thus it is common to use relaxations for copositivity conditions. One simple sufficient condition is: if $S=S^\top$ is positive semidefinite and $N=N^\top$ is elementwise nonnegative then $S+N$ is copositive  (Remark 1.10 in \cite{berman2003completely}). This relaxation is exact for $n_v \le 4$~\cite{diananda1962non} but not for $n_v\ge 5$ (Example 1.30 in \cite{berman2003completely}).  We summarize this relaxation in the following comment:

\vspace{0.05in}
\textbf{Copositive Relaxation:} Let $F(X)$ denote a function mapping some variable $X$ to a matrix $F(X)\in\Sym^m$. We replace the matrix constraint $F(X) \in COP$ by the standard relaxation 
$F(X)-N\succeq 0$ where $N=N^\top\in \R^{m\times m}_{\ge 0}$.
\vspace{0.05in}


Lemma~\ref{lem:ExistingQCs} stated that $\Mc$ contains the existing sets of QCs for ReLU ($\Mset{1}$, $\Mset{2}$, and $\Mset{3}$).  The next lemma notes that these set containments still hold when the copositivity relaxation is used.

\vspace{0.1in}
\begin{lemma}
   \label{lem:ExistingQCsRelaxed}
   Let $\hat{\mathcal{M}}_2$ and $\hat{\mathcal{M}}_c$ denote the subsets of $\Mset{2}$ and $\Mc$ with the copositivity condition replaced by its relaxation.
   
   The sets $\Mset{1}$, $\hat{\mathcal{M}}_2$, and $\Mset{3}$
   are subsets of $\hat{\mathcal{M}}_c$.
\end{lemma}
\begin{proof}
    The proof of Lemma~\ref{lem:ExistingQCs}
    in Appendix~\ref{sec:ProofOfLemExistingQCs}
    actually shows that $\Mset{1}$ and 
    $\Mset{3}$ are subsets of $\hat{\mathcal{M}}_c$.
    
    Moreover, the proof that $\Mset{2}\subset \Mc$  in Appendix~\ref{sec:ProofOfLemExistingQCs}
    can be modified as follows to show that $\hat{\mathcal{M}}_2\subset 
    \hat{\mathcal{M}}_c$.  Consider any $M_2\in \hat{\mathcal{M}}_2$ so that
    \begin{align*}
    M_2 := \bmtx -I & I \\ 0 & I \emtx^\top 
         Q_2 \bmtx -I & I \\ 0 & I \emtx
    \mbox{ with }
    Q_2 = S +N,
    \end{align*}
    where $S$ is positive semidefinite and $N$ is elementwise non-negative.    We need to show that $M_{2,D}$ as defined in \eqref{eq:MD} satisfies the relaxed copositive condition for all $D\in \Dpm^{n_v}$.  Note that $M_{2,D}$ simplifies to:
    \begin{align*}
         M_{2,D} =\bmtx \frac{1}{2} (I-D) \\ \frac{1}{2} (I+D) \emtx^\top 
        Q_2
      \bmtx \frac{1}{2} (I-D) \\ \frac{1}{2} (I+D) \emtx.
    \end{align*}
    The matrices $I_+:=\frac{1}{2}(I+D)$ and  $I_-:=\frac{1}{2}(I-D)$  are diagonal with either $0$ or $1$ along the diagonals.
    Substitute $Q_2=S+N$ to show that $M_{2,D}$ is the sum of a positive semidefinite and elementwise nonnegative matrix. Thus $M_{2,D}$ satisfies the relaxed copositivity conditions for any $D\in \Dpm^{n_v}$ and hence
    $M_2 \in \hat{\mathcal{M}}_c$.
\end{proof}
\vspace{0.1in}

We can use the copositivity relaxation to reformulate \eqref{eq:Optim1} as a semidefinite program (SDP) \cite{boyd2004convex}:
\begin{align}
   \label{eq:Optim2}
   \begin{split}       
   &  \min_{P=P^\top,\, M=M^\top, \, \gamma^2, N_1,\ldots, N_{2^{n_v}}} \gamma^2  \\
   &  P\succeq 0, \,\,\, LMI(P,M,\gamma^2) \prec 0,  \\
   & \bmtx D_i\\ \frac{1}{2} (I+D_i) \emtx^\top 
      M
      \bmtx D_i \\ \frac{1}{2} (I+D_i) \emtx - N_i \succeq 0,
      \,\, i=1,\ldots 2^{n_v} \\
   & N_i =N_i^\top\in \R^{n_v}_{\ge 0},  
     \,\, i=1,\ldots 2^{n_v}. 
   \end{split}   
\end{align}
This SDP has $2^{n_v}$ positive semidefiniteness constraints arising from our relaxation of the copositivity constraints in $\Mc$. This is in addition to the elementwise nonnegativity constraints on $N_i$
for $i=1,\ldots,2^{n_v}$.  This limits the use of $\Mc$ to problems where $n_v$ is relatively small.  Larger values of $n_v$ will require the use of QCs that are subsets of $\Mc$.

\subsection{Numerical Implementation With Subsets of $\Mc$}

One useful subset of $\Mc$ consists of all combinations 
the existing QCs described in Section~\ref{sec:bgReLUQCs}.  These correspond to QCs defined by matrices in the following set:
 \begin{align}
\label{eq:M123}
    \Mset{123} &:=  \left\{ M_1+M_2 +M_3\, : \, 
    M_i\in\Mset{i}, \, i=1,2, 3 \right\}.
\end{align}
The set $\Mc$ is a convex cone (as noted above) and $\Mset{i}\subset \Mc$ ($i=1,2,3$) by Lemma~\ref{lem:ExistingQCs}. It follows that $\Mset{123} \subset\Mc$.

Recall that Theorem~\ref{thm:ReLUQCs2} states that only  ReLU and flipped ReLU satisfy all constraints in the complete set $\Mc$. It is interesting that the proof only uses the QCs defined by matrices in $\Mset{1}$ and $\Mset{2}$.  The proof does not require the use of Zames-Falb QCs defined by $\Mset{3}$. Define another subset $\Mc$ but without the Zames-Falb QCs in $\Mset{3}$:
\begin{align}
\label{eq:M12}
    \Mset{12} &:= \left\{ M_1+M_2 \, : \, 
    M_i\in\Mset{i}, \, i=1,2 \right\}.
\end{align}
The next result states that the Zames-Falb QCs defined by $\Mset{3}$ do not increase the class of QCs  when combined with $\Mset{1}$ and $\Mset{2}$.  

\vspace{0.1in}
\begin{theorem}
    \label{thm:ReLUQCs3}
    The sets $\Mset{12}$ in \eqref{eq:M12} 
    and $\Mset{123}$ in \eqref{eq:M123} 
    are equal, i.e. $\Mset{12}=\Mset{123}$.
\end{theorem}
\vspace{0.05in}
\begin{proof}

\noindent
($\subseteq$) Consider any $M=M_1+M_2\in \Mset{12}$.
Then $M=M_1+M_2+M_3\in \Mset{123}$.
with  $M_3:=0 \in \Mset{3}$.

\vspace{0.1in}
\noindent
($\supseteq$) Consider any $M=M_1+M_2+M_3\in \Mset{12}$.
There exists $Q_1 \in D^{n_v}$, $Q_2 \in COP^{2n_v}$ and $Q_3 \in DH^{n_v}$ such that
\begin{align*}
    M= \bsmtx -I & I \\ 0 & I \esmtx^\top  Q_2 \bsmtx -I & I \\ 0 & I \esmtx + \bsmtx 0 & (Q_1+Q_3)^\top \\ (Q_1+Q_3) & -(2Q_1+Q_3+Q_3^\top) \esmtx.
\end{align*}
Decompose $Q_3=D-P$ where $D\in D^{n_v}$ with $D_{ii}=(Q_3)_{ii}$. All elements of $P=D-Q_3$ are nonnegative because $Q_3$ is  doubly hyperdominant. Next define
$\hat{Q}_1:=Q_1+D \in D^{n_v}$ and $\hat{Q}_2:= Q_2 + \bsmtx 0 & P^\top \\ P & 0 \esmtx$. The term $\bsmtx 0 & P^\top \\ P & 0 \esmtx$ is copositive as it is symmetric and elementwise nonnegative (Remark 1.10 of \cite{berman2003completely}). Hence $\hat{Q}_2$ is copositive.   Finally, define $\hat{M}_1\in \Mset{1}$ and $\hat{M}_2 \in \Mset{2}$ corresponding to $\hat{Q}_1$ and $\hat{Q}_2$, respectively. It can be verified directly that $M=\hat{M}_1+\hat{M}_2 \in \Mset{12}$.
\end{proof}
\vspace{0.1in}

The stability/performance condition in
Theorem~\ref{thm:StabPerfConf}  with $\Mset{12}$ leads to the following SDP:
\begin{align}
   \label{eq:Optim3}
   \begin{split}       
   &  \min_{P=P^\top,\, M=M^\top, \, Q_1, Q_2, N_2, \gamma^2} \hspace{-12mm} \gamma^2  \\
   &  P\succeq 0, \,\,\, LMI(P,M,\gamma^2) \prec 0,  \\
   & M= \bsmtx -I & I \\ 0 & I \esmtx^\top  Q_2 \bsmtx -I & I \\ 0 & I \esmtx + \bsmtx 0 & Q_1 \\ Q_1 & -2Q_1  \esmtx \\
   & Q_1 \in D^{n_v}, \, Q_2-N_2 \succeq 0, \,
   N_2=N_2^\top\in \R_{\ge 0}^{2n_v \times 2n_v}.
   \end{split}   
\end{align}
Here we have used the relaxation described above for the constraint $Q_2\in COP^{2n_v}$. Adding the Zames-Falb QCs defined by $\Mset{3}$ will not reduce the conservatism in this SDP.\footnote{One minor point is that Theorem~\ref{thm:ReLUQCs3} states that $\Mset{12}=\Mset{123}$ where $\Mset{2}$ has a copositivity condition.  The same set equality holds if we define $\Mset{2}$ using the relaxation for the copositivity condition as in SDP \eqref{eq:Optim3}.}
The SDP \eqref{eq:Optim3} will be more computationally efficient, but also more conservative, than the SDP \eqref{eq:Optim2} formulated using $\Mc$. In fact, we provide a numerical example in  Section~\ref{sec:ell2Ex}
where the use of $\Mc$ provides improved results compared to the use of $\Mset{12}$. This numerical example indicates that $\Mset{12}$ is a strict  subset of the complete set $\Mc$.  Thus $\Mc$ should be used when possible as it is the largest possible class of QCs for  ReLU.  

\section{Examples}
\label{sec:Examples}

This section provides two examples to illustrate the complete sets of QCs and incremental QCs for  ReLU.

\subsection{$\ell_2$ Bounds With QCs}
\label{sec:ell2Ex}

We consider the interconnection $F_U(G,\Phi)$ shown in
Figure~\ref{fig:LFTdiagram}  with a  ReLU $\Phi$ wrapped in feedback around the top channels of a nominal system $G$.  The nominal part $G$ is a discrete-time, LTI system \eqref{eq:LTInom} with the following data:
\begin{align*}
    A & := \bsmtx   
    4.1819 &  -4.1122 &   4.1810  & -3.4344 \\
    9.5280 &  -9.1573 &   8.4496  & -6.2574 \\
    8.6800 &  -7.3880 &   6.0327  & -4.0370 \\
    2.8000 &  -1.7500 &   1.2060  & -0.7209 
    \esmtx, \\
    B_1 & := \alpha \times \bsmtx
    9.528  &   8.68  &  5.60  &   2.00 \\
    17.36  &  11.20  &  4.00  &      0 \\
    11.20  &   4.00  &     0  &      0 \\
     2.00  &      0  &     0  &      0
    \esmtx, \,\,
    B_2 := \bsmtx 1 \\ 1 \\ 1 \\ 1\\ 1 \esmtx, \\
    C_1 & := \bsmtx
   -0.5000  &  0.4875 &  -0.2250 &  -0.0250 \\
   -0.4250  &  0.6500 &  -0.6155 &   0.3604 \\
    0.1100  &  0.1282 &  -0.3323 &   0.3064 \\
    0.5645  & -0.5248 &   0.2859 &  -0.0793 
    \esmtx, \,\,
    C_2 := \bsmtx 1 & 1 & 1 & 1 & 1\esmtx, \\
    D_{11} & := \alpha \times \bsmtx
          0  & 0 & 0 & 0 \\
         -1  & 0 & 0 & 0 \\
      -0.85  & -1 & 0 & 0 \\
       0.22  & -0.85 & -1 & 0 \esmtx, \,\,
     D_{12} := 0, \,   D_{21}:=0, \, D_{22}:=0.
\end{align*}
This system is a modification of Example 2 in \cite{carrasco19} using the lifting described in \cite{noori24ReLURNN}. The  ReLU $\Phi$ has dimension $n_v=4$ and $\alpha$ is a gain variation included on the $(v,w)$ channels associated with $\Phi$. If $\alpha=0$ then dynamics from $d$ to $e$ are governed by the "nominal" LTI system $G_0$ described by $(A,B_2,C_2,D_{22})$.  The induced $\ell_2$ gain in this case is equal to the $H_\infty$ norm: $\|G_0\|_\infty=39.8$.

We can use Theorem~\ref{thm:StabPerfConf} to compute an upper bound on the induced $\ell_2$ gain of $F_U(G,\Phi)$ when $\alpha\ge 0$.  We compute bounds using the sets of QCs described by $\Mc$ and $\Mset{12}$ and the corresponding SDPs   \eqref{eq:Optim2} and \eqref{eq:Optim3}, respectively. Figure~\ref{fig:RNNgainPlot} shows the two bounds for 20 linearly spaced values of $\alpha$ from $0$ to $0.6$.  Both curves agree with the nominal gain $\|G_0\|_\infty=39.8$ at $\alpha=0$. The complete set $\Mc$ provides a less conservative bound than the QCs 
defined by $\Mset{12}$. It took 17.5sec and 22.1 sec to compute all 20 points on the curves for $\Mset{12}$ and $\Mc$, respectively. Thus the computational costs are similar for this example although the computation with $\Mc$ to grow more rapidly with $n_v$. The curve with $\Mset{123}$ was also computed but is not shown since it is indistinguishable from the curve for $\Mset{12}$.  This is expected since Theorem~\ref{thm:ReLUQCs3} states $\Mset{12}=\Mset{123}$.

\begin{figure}[h!]
  \centering
  \includegraphics[width=0.42\textwidth]{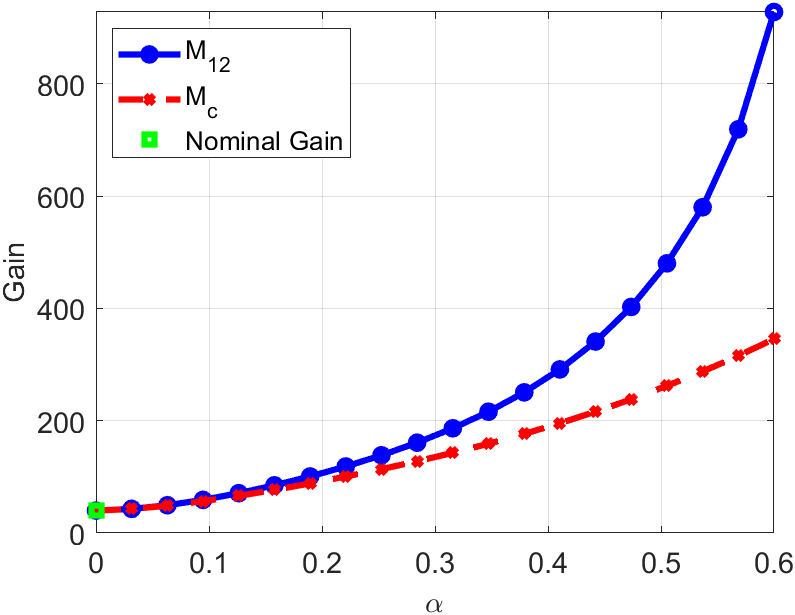}
  \caption{Bound on induced $\ell_2$ gain for system $F_U(G,\Phi)$ vs. $\alpha$ using  QCs defined by 
  $\Mc$ and $\Mset{12}$. We expect the interconnection to eventually become unstable as $\alpha$ increases which is consistent with both curves. The complete set $\Mc$ provides a less conservative (smaller) bound on the gain.  }
\label{fig:RNNgainPlot}
\end{figure}


The parameter $\alpha$ appears in $(B_1,D_{11})$ and thus it scales the effect of the ReLU as it feeds back into $G$.
We expect that the interconnection may become unstable as $\alpha$ increases which is consistent with both curves shown in Figure~\ref{fig:RNNgainPlot}. In fact, 
the SDP \eqref{eq:Optim3} with $\Mset{12}$ becomes infeasible for $\alpha \ge 0.699$.  The corresponding curve (blue solid) in Figure~\ref{fig:RNNgainPlot} becomes unbounded if we extend  the horizontal axis to $0.699$ and we cannot prove stability for larger values of $\alpha$. The SDP  \eqref{eq:Optim3} with the complete set $\Mc$ remains feasible up to $\alpha\approx 1.1016$ and becomes infeasible for larger values of $\alpha$. Thus with the complete set: (i) we are able to prove stability over a larger range of $\alpha$, and (ii) provide smaller (less) conservative bounds on the $\ell_2$ gain for values of $\alpha$ for which the system is stable.

\subsection{Lipschitz Bounds With Incremental QCs}
\label{sec:LipschitzEx}

This section provides a simple example to bound the Lipschitz constant of a simple Neural Network (NN).   We use the LipSDP condition \cite{fazlyab2019efficient} but combined with our complete set of 
incremental QCs for  ReLU.  Consider
an $\ell$-layer NN $f:\R^{n_0}\to\R^{n_{\ell+1}}$ defined by weights $W^k \in \R^{n_{k+1}\times n_k}$
and biases $b^k \in \R^{n_{k+1}}$. The NN maps $x \in R^{n_0}$ to $y=f(x)\in \R^{n_{\ell+1}}$ by:
\begin{align}
\begin{split}    
    x^0 & = x \\
    x^{k+1} & = \Phi^k( W^k x^k + b^k) \mbox{ for } k = 0,\ldots, \ell-1 \\
    y & = W^\ell x^\ell + b^\ell.
\end{split}
\end{align}
Here we consider the case where $k^{th}$ activation function $\Phi^k$ is a  ReLU of dimension $n_{k+1}$. $L$ is a Lipschitz bound (in the $2$-norm) for the NN i $\|f(x)-f(\hat{x})\|_2 \le L \|x-\hat{x}\|_2$ for all $x,\hat{x} \in \R^{n_0}$.   

We can stack 
the activation functions into a single
 ReLU $\Phi:\R^m\to \R^m$  with total dimension $m:=\sum_{k=1}^{\ell} n_k$.
The NN can be expressed compactly as
\begin{align}
B x & = \Phi(A x + b) \\
y  & = C x + b^\ell,
\end{align}
where\footnote{blkdiag($\cdots$) denotes the block diagonal augmentation of matrices.} 
\begin{align*}
    & x : = \bmtx x^0 \\ \vdots \\  x^\ell \emtx, 
    \,\,\,
    b : = \bmtx b^0 \\ \vdots \\ b^{\ell-1} \emtx, \\
    & A : = \bmtx \mbox{blkdiag}(W^0,\ldots,W^{\ell-1}) & 0_{m\times n_\ell} \emtx, 
    \,\,\,
    B : = \bmtx 0_{m\times n_0} & I_m \emtx, \\
    & C := \bmtx 0_{n_{\ell+1}\times (n_0+m-n_\ell)} & W^\ell \emtx,
    \,\,\,
    D = \bmtx I_{n_0} & 0_{n_0\times m} \emtx
\end{align*}

It follows from Theorem 2 in \cite{fazlyab2019efficient} that $L$ is a Lipschitz bound for $f$ if there exists
an incremental QC, defined by $M$, for $\Phi$ such that the following LMI holds:
\begin{align}
    LMI_{Lip}(M,L):=
\bmtx A \\ B \emtx^\top M \bmtx A \\ B \emtx
+ C^\top C -L^2\, D^\top D \preceq 0
\end{align}
We can find the smallest Lipschitz bound by searching over convex classes of incremental QCs. An SDP can be formulated to
minimize $L^2$ subject to 
$LMI_{Lip}(M,L)\preceq 0$ with  $\Miset{1}$  or the relaxed version of the complete set $\Mic$ using similar steps to those given in Section~\ref{sec:numimp}.

We used the SDP condition to compute Lipschitz bounds for a small two-layer NN with the following weights
\begin{align}
   & W^0 = \bmtx 
   -0.575 &   0.420 &   0.050 \\
   -0.730 &   0.200 &  -1.020 \emtx, \\
   & W^1 = \bmtx   1.120 &   -0.630 \emtx, 
   \,\,
   W^2 = \bmtx -0.700 \\ -1.300 \emtx.
\end{align}
The biases are not specified since they do not enter into the LMI condition. 

We obtain a Lipschitz bound $L = 1.2528$ when searching over $\Miset{1}$.  We also obtain $L=1.2528$ when searching over both $\Miset{1}$ and our new set $\Miset{2}$. Thus $\Miset{2}$ does not improve the bound on this specific example. However,  the bound is improved to $L = 1.1817$ when we use the complete set of incremental constraints $\Mic$. The times to compute the bounds with $\Miset{1}$, $\Miset{1}+\Miset{2}$, and $\Mic$ was 0.3sec, 0.3 sec, and 3.6sec, respectively. 

A lower bound on the Lipschitz constant was obtained by randomly sampling $N=10^7$ pairs of NN inputs $(x,\hat{x})$ from a zero-mean, unit variance Gaussian distribution and computing the corresponding NN outputs $(y,\hat{y})$. The largest value of the ratio 
$\frac{\|y-\hat{y}\|_2}{\|x-\hat{x}\|_2}$ over all the samples (assuming $x\ne\hat{x}$) is a lower bound on the Lipschitz constant. This sampled lower bound is $1.1817$. This matches the upper bound with $\Mic$ up to the reported digits.

This is an academic example illustrating that it is possible to improve upon the Lipschitz bounds obtained only using $\Miset{1}$. Importantly, the complete set $\Mic$ provides a unifying view that may aid in the search of additional useful subsets of incremental QCs.
More realistic NNs would have activation functions where the total dimension $m$ can be very large, requiring future study on possible combinations of our proposed method and existing SDP scaling techniques \cite{wang2024scalability}.


\section{Conclusions}

This paper derived a complete set of quadratic constraints (QCs) for the  ReLU.  The complete set of QCs is described by a collection of $2^{n_v}$ matrix copositivity conditions where $n_v$ is the dimension of the  ReLU. The relationship between our complete set and existing QCs has been carefully discussed. We also derived a similar complete set of incremental QCs for  ReLU.
We illustrate the use of the complete set of QCs to assess stability and performance for recurrent neural networks with ReLU activation functions.  We will study, as future work, the conservatism (if any) and the scalability of using the complete set of  ReLU QCs for stability/performance analysis.


\section{Acknowledgments}

The authors acknowledge AFOSR Grant \#FA9550-23-1-0732 for funding of this work. The authors also acknowledge Carsten Scherer for asking the question about the complete set of ReLU quadratic constraints.

\bibliographystyle{IEEEtran}
\bibliography{references} 

\appendix
\vspace{0.1in}

\subsection{Proof of Lemma~\ref{lem:ExistingQCs}}
\label{sec:ProofOfLemExistingQCs}

\begin{repeatlem}{lem:ExistingQCs}
The sets $\Mset{1}$, $\Mset{2}$, and $\Mset{3}$ are subsets of $\Mc$.
\end{repeatlem}
\begin{proof}
We need to show that any $M_i\in \Mset{i}$ ($i=1,2,3$) satisfies the condition:  $M_{i,D}$ is copositive for all $D\in \Dpm^{n_v}$.

First, consider any $M_1\in \Mset{1}$ so that $M_1 := \bsmtx 0 & Q_1 \\ Q_1 & -2Q_1 \esmtx$
for some  $Q_1 \in D^{n_v}$. If $D\in \Dpm^{n_v}$ then
\begin{align*}
     M_{1,D} & :=\bmtx D \\ \frac{1}{2} (I+D) \emtx^\top 
      M_1
      \bmtx D \\ \frac{1}{2} (I+D) \emtx \\
      & = \frac{1}{2} D Q_1 (I+D) + \frac{1}{2}  (I+D) Q_1 D
       - \frac{1}{2} (I+D) Q_1 (I+D).
\end{align*}
All the matrices in this expression are diagonal so this simplifies to $M_{1,D} = \frac{1}{2} Q_1 (D^2-I)$.  Finally, $D^2=I$ so that $M_{1,D}=0$. Thus $M_{1,D}$ is (trivially) copositive for any $D\in \Dpm^{n_v}$ and hence $M_1\in \Mc$.

Next, consider any $M_2\in \Mset{2}$ so that
\begin{align*}
M_2 := \bmtx -I & I \\ 0 & I \emtx^\top 
     Q_2 \bmtx -I & I \\ 0 & I \emtx
\mbox{ with }
Q_2 \in COP^{2n_v}.
\end{align*}
Again, define $M_{2,D}\in \R^{n_v\times n_v}$  for some $D\in \Dpm^{n_v}$.  Then $M_{2,D}$ simplifies to:
\begin{align*}
     M_{2,D} =\bmtx \frac{1}{2} (I-D) \\ \frac{1}{2} (I+D) \emtx^\top 
      Q_2
      \bmtx \frac{1}{2} (I-D) \\ \frac{1}{2} (I+D) \emtx.
\end{align*}
The matrices $I_+:=\frac{1}{2}(I+D)$ and  $I_-:=\frac{1}{2}(I-D)$  are diagonal with either $0$ or $1$ along the diagonals.
Therefore, 
\begin{align}
    \bar{v}^\top M_{2,D} \bar{v}
    = \bmtx I_- \bar{v} \\ I_+ \bar{v} \emtx^\top
    Q_2 
    \bmtx I_- \bar{v} \\ I_+ \bar{v} \emtx
    \ge 0 
    \hspace{0.2in} \forall \bar{v}\in\R_{\ge 0}^{n_v}.
\end{align}
The product is nonnegative because $Q_2$ is copositive and both $I_-\bar{v}$, $I_+\bar{v}$ are in the nonnegative orthant $\R^{n_v}_{\ge 0}$ when $\bar{v}\in\R_{\ge 0}^{n_v}$.
Thus $M_{2,D}$ is copositive
for any $D\in \Dpm^{n_v}$ and hence $M_2\in \Mc$.

Finally, consider any $M_3 \in \Mset{3}$ so that $M_3:= \bsmtx 0 & Q_3 \\ Q_3^\top & -(Q_3+Q_3^\top) \esmtx$ for
some $Q_3\in DH^{n_v}$. Again, define $M_{3,D}\in \R^{n_v\times n_v}$  for some $D\in \Dpm^{n_v}$. 
Then $M_{3,D}$ simplifies to:
\begin{align}
\label{eq:M3D}
\begin{split}    
     M_{3,D} & = -\frac{1}{4} (I-D) Q_3^\top (I+D) 
     - \frac{1}{4} (I+D) Q_3 (I-D) \\
         & = -I_- Q_3^\top I_+ - I_+ Q_3 I_- .
\end{split}
\end{align}
where $I_+:=\frac{1}{2}(I+D)$ and  $I_-:=\frac{1}{2}(I-D)$  
as above.  These diagonal matrices are complementary in the sense that they sum to the identity. We can assume $I_+ = \bsmtx I & 0 \\ 0 & 0 \esmtx$ and $I_- = \bsmtx 0 & 0 \\ 0 & I \esmtx$ by properly permuting the rows and columns. Partition $Q_3$ conformably with $I_+$ and $I_-$ so that the product in \eqref{eq:M3D} is:
\begin{align*}
Q_3 = \bmtx (Q_3)_{11} & (Q_3)_{12} \\ (Q_3)_{21} & (Q_3)_{22} \emtx
\Rightarrow
     M_{3,D} & = \bmtx 0 & -(Q_3)_{12} \\ -(Q_3)_{12}^\top & 0 \emtx.
\end{align*}
Every entry of the block $(Q_3)_{12}$ is non-positive because $Q_3$ is doubly hyperdominant.  Hence every entry of $M_{3,D}$ is non-negative, i.e. $M_{3,D} \in \R_{\ge 0}^{n_v \times n_v}$.  Symmetric matrices that are elementwise nonnegative are copositive (Remark 1.10 of \cite{berman2003completely}).
Thus $M_{3,D}$ is copositive
for any $D\in \Dpm^{n_v}$ and hence $M_3\in \Mc$.
\end{proof}
\vspace{0.1in}





\subsection{Proof of Lemma~\ref{lem:DHdecomp}
}

\label{sec:ProofOfDHdecomp}

\begin{repeatlem}{lem:DHdecomp}
Let $T_2=T_2^\top \in DH^{n_v}$ be given
with $\sum_{k=1}^{n_v} T_{ik}=0$ for $i=1,\ldots,n_v$.    Then there exists
$\{ \lambda_{ij} \}_{i,j=1}^{n_v} \in \R_{\ge 0}$ such that
\begin{align}
   \label{eq:T2decompApp}
   T_2 = \sum_{i,j=1}^{n_v} \lambda_{i,j} (e_i-e_j)(e_i-e_j)^\top.
\end{align}
\end{repeatlem}
\vspace{0.05in}
\begin{proof}
First, let $r:=\max_{ij} (T_2)_{ij}$. Define $R:=\frac{1}{r}(rI-T_2)$ so that $T_2 = r( I-R)$. $R$ is symmetric, has all non-negative entries, and its rows/columns sum to 1, i.e. $R$ is a symmetric doubly stochastic matrix.  Every such matrix
can be decomposed as $R=\sum_k \frac{\alpha_k}{2} (P_k + P_k^\top)$
with $\alpha_k\ge 0$, $\sum_k \alpha_k=1$, and $P_k$ are permutation matrices~\cite{katz1970extreme,cruse1975note}.
Thus $T_2$ can be decomposed as
$T_2 = \sum_k  \beta_k (2I - P_k - P_k^\top)$ with $\beta_k = \frac{r\alpha_k}{2} \ge 0$.

Next, each permutation can be expressed as $P_k = \sum_{i=1}^{n_v} e_i e_{\pi_k(i)}^\top$ where $\pi_k$ is a permutation function that maps $\{1,\ldots,n_v\}$ one-to-one onto $\{1,\ldots,n_v\}$. Thus each term $2I-P_k-P_k^\top$ can be further decomposed as:
\begin{align*}
2I-P_k-P_k^\top 
& = \sum_{i=1}^{n_v} 2e_ie_i^\top - e_i e_{\pi_k(i)}^\top
-e_{\pi_k(i)} e_i^\top \\
& = \sum_{i=1}^{n_v} (e_i-e_{\pi_k(i)}) (e_i-e_{\pi_k(i)})^\top.  
\end{align*}
Substitute this into 
$T_2 = \sum_k  \beta_k (2I - P_k - P_k^\top)$ 
to get a decomposition of the form shown in
\eqref{eq:T2decompApp}.
\end{proof}
\vspace{0.1in}

\end{document}